\documentclass[11pt, a4paper, logo, copyright, ]{deepmind}
\usepackage[authoryear, sort&compress, round]{natbib}
\usepackage{url}

\usepackage[decisionutilitycolor]{influence-diagrams}

\usepackage{amsmath}    
\usepackage{amssymb}    
\usepackage{amsthm}
\usepackage{amsfonts}
\usepackage{thmtools}
\usepackage{thm-restate}
\usepackage{bm}         
\usepackage{booktabs}   
\usepackage{enumitem}   
\usepackage[utf8]{inputenc}  
\usepackage{mathtools}  
\usepackage{microtype}  
\usepackage{chngcntr}
\usepackage{multirow}
\usepackage{natbib} 
\usepackage{tcolorbox}  
\usepackage{xcolor}     
\usepackage{subcaption}

\usepackage[capitalize]{cleveref}
\usepackage{nameref}
\crefname{assumption}{Assumption}{Assumptions}

\usepackage{symbols}

\theoremstyle{plain}

\newtheorem{theorem}{Theorem}
\newtheorem{lemma}{Lemma}

\theoremstyle{definition}
\newtheorem{definition}{Definition}
\newtheorem{assumption}{Assumption}
\theoremstyle{plain}

\newtheorem*{example*}{Example continued}

\usetikzlibrary{matrix}
\tikzset{ 
table/.style={
  matrix of nodes,
  row sep=-\pgflinewidth,
  column sep=-\pgflinewidth,
  nodes={rectangle,text width=10em,align=center},
  text depth=1.25ex,
  text height=2.5ex,
  nodes in empty cells
},
row 1/.style={nodes={fill=black!10,text depth=0.4ex,text height=2ex}},
}

\def\appendixname{}

\begin{document}

\title{Discovering Agents}

\begin{abstract}
Causal models of agents have been used to analyse the safety aspects of machine learning systems.
But identifying agents is non-trivial -- often the causal model is just assumed by the modeler without much justification -- and modelling failures can lead to mistakes in the safety analysis. 
This paper proposes the first formal causal definition of agents -- roughly that agents are systems that would adapt their policy if their actions influenced the world in a different way.
From this we derive the first causal discovery algorithm for discovering agents from empirical data, and give algorithms for translating between causal models and game-theoretic influence diagrams.  
We demonstrate our approach by resolving some previous confusions caused by incorrect causal modelling of agents.
\end{abstract}

\author[1]{Zachary Kenton}
\author[1]{Ramana Kumar}
\author[2]{Sebastian Farquhar}
\author[1]{Jonathan Richens}
\author[3]{Matt MacDermott}
\author[1]{Tom Everitt}
    
\correspondingauthor{zkenton@deepmind.com}
    
\affil[1]{DeepMind}
\affil[2]{University of Oxford, work begun while author was at DeepMind}
\affil[3]{Imperial College London}
    
\reportnumber{} 
\renewcommand{\today}{}
    
\maketitle

\section{Introduction}
\label{sec:intro}

How can we recognise agents? In economics textbooks, certain entities are clearly delineated as choosing actions to maximise utility. In the real world, however, distinctions often blur. Humans may be almost perfectly agentic in some contexts, while manipulable like tools in others. Similarly, in advanced reinforcement learning (RL) architectures, systems can be composed of multiple non-agentic components, such as actors and learners, and trained in multiple distinct phases with different goals, from which an overall goal-directed agentic intelligence emerges.

It is important that we have tools to discover goal-directed agents. 
Artificially intelligent agents that competently pursue their goals might be dangerous depending on the nature of this pursuit, because goal-directed behaviour can become pathological outside of the regimes the designers anticipated \citep{bostrom2017superintelligence,yudkowsky2008artificial}, and because they may pursue convergent instrumental goals, such as resource acquisition and self-preservation \citep{omohundro2008the}. 
Such safety concerns motivate us to develop a formal theory of goal-directed agents, to facilitate our understanding of their properties, and avoid designs that pose a safety risk.

The central feature of agency for our purposes is that agents are systems whose outputs are \emph{moved by reasons} \citep{dennett1987intentional}. In other words, the reason that an agent chooses a particular action is that it ``expects it'' to precipitate a certain outcome which the agent finds desirable. For example, a firm may set the price of its product to maximise profit. This feature distinguishes agents from other systems, whose output might accidentally be optimal for producing a certain outcome. For example, a rock that is the perfect size to block a pipe is accidentally optimal for reducing water flow through the pipe.

Systems whose actions are moved by reasons, are systems that would act differently if they ``knew'' that the world worked differently. For example, the firm would be likely to adapt to set the price differently, if consumers were differently price sensitive (and the firm was made aware of this change to the world). In contrast, the rock would not adapt if the pipe was wider, and for this reason we don't consider the rock to be an agent.

Behavioural sensitivity to environment changes can be modelled formally with the language of causality and structural causal models (SCMs) \citep{pearl2009causality}.  
To this end, our first contribution is to introduce \emph{mechanised} SCMs (\cref{sec:mechanised-scms,sec:edge-labelled}), a variant of mechanised causal games \citep{hammond2021reasoning}, and give an algorithm which produces its graph given the set of interventional distributions (\cref{sec:discovering-edge-labelled-mcgs}). 
Building on this, our second contribution is an algorithm for determining which variables represent agent decisions and which represent the objectives those decisions optimise (i.e., the \emph{reasons that move the agent}), see \cref{sec:discover-game-graphs}. 
This lets us convert a mechanised SCM into a (structural) causal game \citep{hammond2021reasoning}\footnote{We can also reverse this, converting a causal game into a mechanised SCM (\cref{sec:mech-id}).}. 
Combined, this means that under suitable assumptions, we can infer a game graph from a set of experiments, and in this sense \emph{discover agents}.
Our third contribution is more philosophical, giving a novel formal definition of agents based on our method, see \cref{sec:other-characterisations-of-agents}. 

These contributions are important for several reasons. First, they ground game graph representations of agents in causal experiments.
These experiments can be applied to real systems, or used in thought-experiments to determine the correct game graph and resolve confusions (see \cref{sec:examples}). With the correct game graph obtained, the researcher can then use it to understand the agent’s incentives and safety properties \citep{everitt2021agent,halpern2018towards}, with an extra layer of assurance that a modelling mistake has not been made. 
Our algorithms also 
open a path to automatic inference of game graphs, especially in situations where experimentation is cheap, such as in software simulations.

\begin{figure}
    \begin{subfigure}[b]{0.25\textwidth}
        \centering
        \includegraphics[width=\textwidth]{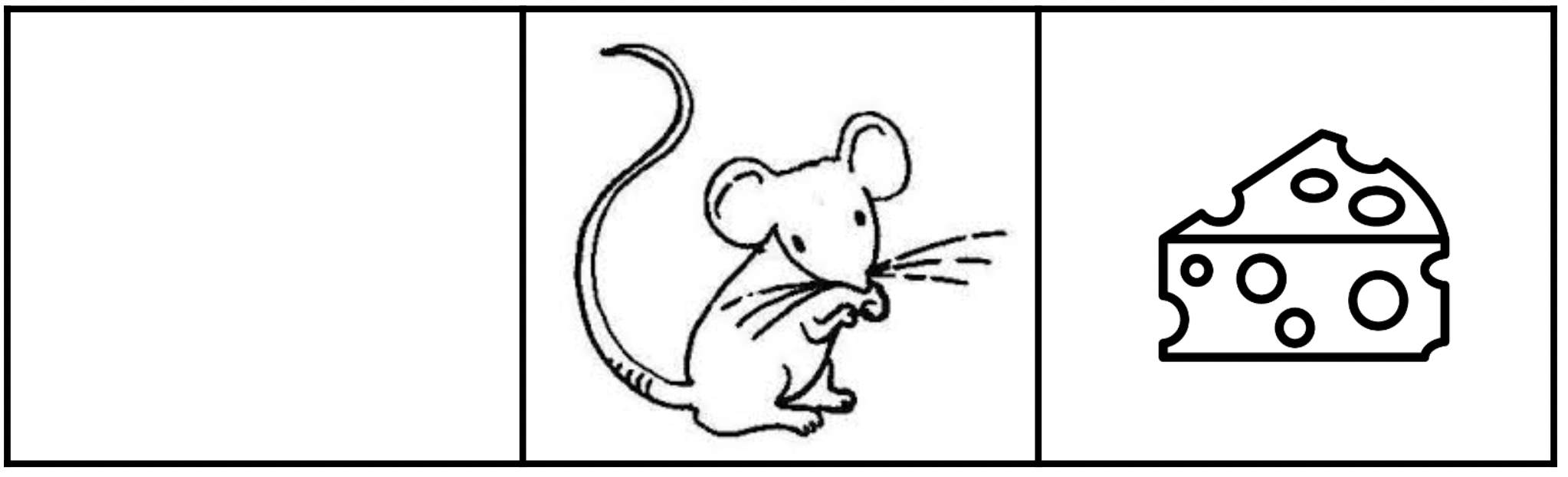}
        \caption{Gridworld}
        \label{fig:mouse-gridworld}
    \end{subfigure}
    \begin{subfigure}[b]{0.25\textwidth}
        \centering
        \begin{influence-diagram}

          \node (D) [decision, player1] {$\decisionvar$};
    
          \node (X) [right = of D] {$\structvar$};
    
          \node (U) [right = of X, utility, player1] {$\utilvar$};

          \edge {D} {X};
    
          \edge {X} {U};

        \end{influence-diagram}
        \caption{Game graph, $\macid$}
        \label{subfig:mouse-mascid-real}
    \end{subfigure}
    \begin{subfigure}[b]{0.25\textwidth}
        \centering
        \begin{influence-diagram}

          \node (D) {$\decisionvar$};
    
          \node (X) [right = of D] {$\structvar$};
    
          \node (U) [right = of X] {$\utilvar$};
          
          \node (MD) [above = of D, mechanism] {$\mecvar{\decisionvar}$};
    
          \node (MX) at (X|-MD) [mechanism] {$\mecvar{\structvar}$};
    
          \node (MU) at (U|-MD) [mechanism] {$\mecvar{\utilvar}$};

          \edge {D} {X};
    
          \edge {X} {U};

          \path (MD) edge[function-edge] (D);
          \path (MX) edge[function-edge] (X);
          \path (MU) edge[function-edge] (U);

          \path (MU) edge[terminal-edge, bend right=27] (MD);

          \path (MX) edge[mechanism-edge] (MD);

        \end{influence-diagram}
        \caption{Mech.\ causal graph, $\xscg$}
        \label{subfig:mouse-xcg-model}
    \end{subfigure}
    \begin{subfigure}[b]{0.23\textwidth}
     \begin{influence-diagram}
    \cidlegend{
  \legendrow              {}  {chance} \\
  \legendrow              {mechanism} {mechanism} \\
  \legendrow              {decision,player1}  {decision}\\
  \legendrow              {utility,player1}   {utility}\\
  \legendrow[causal]      {draw=none} {terminal} \\
  \legendrow[information] {draw=none} {non-terminal} }
\edge[terminal-edge] {causal.west} {causal.east};
\edge[mechanism-edge] {information.west} {information.east};
    \end{influence-diagram}
    \end{subfigure}

    \caption{Different graphical representations for the the mouse example (\cref{sec:intro-example}).
    }
    \label{fig:mouse}
\end{figure}

\subsection{Example}
\label{sec:intro-example}
To illustrate our method in slightly more detail,
consider the following minimal example, consisting of a gridworld with three squares, and with a mouse starting in the middle square (\cref{fig:mouse-gridworld}).
The mouse can go either left or right, represented by binary variable $D$.  
There is some ice which may cause the mouse to slip: the mouse's position, $X$, follows its choice, $D$, with probability $p=0.75$, and slips in the opposite direction with probability $1-p$.
Cheese is in the right square with probability $q=0.9$, and the left square with probability $1-q$. The mouse gets a utility, $U$, of $1$ for getting the cheese, and zero otherwise.
The directed edges $D\to X$ and $X\to U$ represent direct causal influence.

The decision problem can be represented with the game graph in \cref{subfig:mouse-mascid-real}: the agent makes a decision, $D$, which affects its position, $X$, which affects its utility, $U$.
The intuition that the mouse would choose a different behaviour for other settings of the parameters $p$ and $q$, can be captured by a  \emph{mechanised causal graph} (\cref{subfig:mouse-xcg-model}).
This graph contains additional \emph{mechanism nodes}, $\mecvar{D}, \mecvar{X}, \mecvar{U}$ in black, representing the mouse's decision rule and the parameters $p$ and $q$.
As usual, edges between mechanisms represent direct causal influence, 
and show that if we intervene to change the cheese location, say from $q=0.9$ to $q=0.1$, and the mouse is aware\footnote{The mouse could become \emph{aware} of this through learning from repeated trials under soft interventions of $X$ and $U$ which occur on every iteration, see \cref{sec:mechanised-scms} for further discussion.}
of this, then the mouse's decision rule changes (since it's now more likely to find the cheese in the leftmost spot).
Experiments that change $p$ and $q$ in a way that the mouse is aware of, generate interventional data that can be used to infer both the mechanised causal graph (\cref{subfig:mouse-xcg-model}) and from there the game graph (\cref{subfig:mouse-mascid-real}). 
The edge labels (colours) in \cref{subfig:mouse-xcg-model} will be explained in \cref{sec:edge-labelled}.

\subsection{Other Characterisations of Agents}
\label{sec:other-characterisations-of-agents}

To put our contribution in context, we here describe previous characterisations of agents:
\begin{itemize}
    \item \emph{The intentional stance}: an agent's behaviour can be usefully understood as trying to optimise an objective \citep{dennett1987intentional}.
    \item \emph{Cybernetics}: an agent's behaviour adapts to achieve an objective (e.g.\ \citealp{Wiener1961,Ashby1956}).
    \item \emph{Decision theory / game theory / economics / AI}: An agent selects a policy to optimise an objective.
    \item An agent is a system whose behaviour can be \emph{compressed} with respect to an objective function \citep{orseau2018agents}.
    \item ``An \emph{optimising system} is ... a part of the universe [that] moves predictably towards a small set of target configurations'' \citep{flint2020the}.
    \item A \emph{goal-directed system} has self-awareness, planning, consequentialism, scale, coherence, and flexibility \citep{ngo2020agi}.
    \item Agents are ways for the future to influence the past (via the agent's model of the future) \citep{garrabrant2021saving,Foerster1951}.
\end{itemize}
For broader reviews, see also \citet{shimi2021lit} and \cite{wooldridge1995}.

Our proposal can be characterised as: \emph{agents are systems that would adapt their policy if their actions influenced the world in a different way.}
This may be read as an alternative to, or an elaboration of, the intentional stance and cybernetics definitions (depending on how you interpret them) couched in the language of causality and counterfactuals.
Our definition is fully consistent with the decision theoretic view, as agents choose their behaviour differently depending on its expected consequences, but doesn't require us to know who is a decision maker in advance, nor what they are optimising.

The formal definition by \citeauthor{orseau2018agents} can be viewed as an alternative interpretation of the intentional stance: the behaviour of systems that choose their actions to optimise an objective function should be highly compressible with respect to that objective function.
However, \citeauthor{orseau2018agents}'s definition suffers from two problems:
First, in simple settings, where there is only a small and finite number of possible behaviours (e.g.\ the agent decides a single binary variable), it will not be possible to compress any policy beyond its already very short description.
Second, the compression-based approach only considers what the system actually does.
It may therefore incorrectly classify as agents systems with accidentally optimal input-output mappings, such as the water-blocking rock above.
Our proposal avoids these issues, as even a simple policy may adapt, but the rock will not.

The insightful proposal by \citeauthor{flint2020the} leaves open the question of what part of an optimising system is the agent, and what part is its environment.
He proposes the additional property of \emph{redirectability}, but its not immediately clear how it could be used to identify decision nodes in a causal graph (intervening on almost any node will change the outcome-distribution).

The goal-directed systems that \citeauthor{ngo2020agi} has in mind are agentic in a much stronger sense than we are necessarily asking for here, and each of the properties contain room for interpretation.
However, our definition is important for goal-directedness, as it distinguishes incidental influence that a decision might have on some variable, from more directed influence: only a system that counterfactually adapts can be said to be trying to influence the variable in a systematic way. Counterfactual adaptation can therefore be used as a \emph{test for goal-directed influence}.

Our definition also matches closely the backwards causality definition of agency by \citet{garrabrant2021saving}, as can be seen by the time-opposing direction of the edges $\mecvar{X}\to \mecvar{D}$ and $\mecvar{U}\to\mecvar{D}$ in \cref{subfig:mouse-xcg-model}.
It also fits nicely with formalisations of agent incentives \citep{halpern2018towards,everitt2021agent}, which effectively rely on behaviour in counterfactual scenarios of the form that we consider here.
This is useful, as a key motivation for our work is to analyse the intent and incentives of artificial agents.

\subsection{What do we consider an agent}

Before digging into the mathematical details of our proposal, let us make some brief remarks about what it considers an agent and not.
From a pre-theoretic viewpoint, humans might be the most prototypical example of an agent.
Our method reliably classifies humans as agents, because humans would usually adapt their behaviour if suitably informed about changes to the consequences of their actions.
It's also easy to communicate the change in action-consequences to a human, e.g.\ using natural language.
But what about border-line cases like thermostats or RL agents?
Here, the answer of our definition depends on whether one considers the \emph{creation process} of a system when looking for adaptation of the policy. 
Consider, for example, changing the mechanism for how a heater operates, so that it cools rather than heats a room. 
An existing thermostat will not adapt to this change, and is therefore not an agent by our account. 
However, if the designers were aware of the change to the heater, then they would likely have designed the thermostat differently.
This adaptation means that the thermostat \emph{with its creation process} is an agent under our definition.
Similarly, most RL agents would only pursue a different policy if retrained in a different environment.
Thus we consider the system of the \emph{RL training process} to be an agent, but the \emph{learnt RL policy} itself, in general, won’t be an agent according to our definition (as after training, it won’t adapt to a change in the way its actions influence the world, as the policy is frozen).

For the purpose of detecting goal-directed behaviour, the relevant notion of agency often includes the creation process.
Being forced to consider the creation process of the system, rather than just the system itself, may seem inconvenient.
However, we consider it an important insight that simple forms of agents often derive much of their goal-directedness from the process that creates them.

\subsection{Outline}
Our paper proceeds as follows: we give relevant technical background in \cref{sec:background}; give our main contribution, algorithms for discovering agents, in \cref{sec:alg-disc-agents}; show some example applications of this in \cref{sec:examples} followed by a discussion in \cref{sec:discussion}.

\section{Background}
\label{sec:background}
Before we get to our algorithms for discovering agents, we cover some necessary technical background. The mathematical details can be found in \cref{app:math-background}.
Throughout, random variables are represented with roman capital letters (e.g.\ $V$), and their outcomes with lower case letters (e.g.\ $v$). We use bold type to indicate vectors of variables, $\evars$, and vectors of outcomes $\evals$. 
For simplicity, each variable $V$ only has a finite number of possible outcomes, denoted $\dom(V)$.
For a set of variables, $\dom(\bm{V}) = \prod_{V\in\bm{V}}\dom(V)$.

In structural causal models (SCMs; \citealp{pearl2009causality}), randomness comes from exogenous (unobserved) variables, $\exovars$, whilst deterministic structural equations relate endogenous variables, $\evars$, to each other and to the exogenous ones, i.e., $\evar = f^\evar(\evars, \exovar^{\evar})$ (\cref{def:scm}).
An SCM $M$ induces a causal graph $G$ over the endogenous variables, in which there is an edge  $W\to V$ if $f^V(\evars, \exovar^{\evar})$ depends on the value of $W$ (\cref{def:cg}).
The SCM is \emph{cyclic} if its induced graph is, and \emph{acyclic} otherwise. We never permit self-loops $V\to V$.
Parents, children, ancestors and descendants in the graph are denoted $\Pa^\evar$, $\Ch^\evar$, $\Anc^\evar$, and $\Desc^\evar$, respectively
(neither include the variable $\evar$).
The family is denoted by $\Fa^\evar = \Pa^\evar \cup \{\evar\}$. 
Interventions on $\bm{Y}\subseteq \evars$, denoted $\doo(\bm{Y} = \bm{y})$,
can be realised as replacements of a subset of structural equations,
so that $\bm{Y} = \bm{f}^{\bm{Y}}(\evars, \exovar^{\bm{Y}})$ gets replaced with $\bm{Y} = \bm{y}$ (\cref{def:submodel}).
The joint distribution $P(\evars \mid \doo(\bm{Y}=\bm{y}))$ 
is called the \emph{interventional distribution} associated with intervention $\doo(\bm{Y}=\bm{y})$.
A \emph{soft} intervention instead replaces $\bm{f}^{\bm{Y}}$ with some other (potentially non-constant) functions $\bm{g}^{\bm{Y}}$.

A (structural) causal game \citep{everitt2021agent,hammond2021reasoning,Koller2003-yh} is similar to an SCM, but where the endogenous variables are partitioned into chance, decision, and utility variables (denoted $\structvars$, $\decisionvars$ and $\utilvars$ respectively), and for which the decision variables have no structural equations specified. Instead a decision-maker is free to choose a probability distribution over actions $D$, given the information revealed by the outcome of the parents of $D$ (\cref{def:scim}).
The decision variables belonging to agent $A$ are denoted $\decisionvars^A\subseteq \decisionvars$, and the agent's utility is taken to be the sum of the agent's utility variables, $\utilvars^A\subseteq \utilvars$.
A collection of decision rules for all of a player's decisions is called a \emph{policy}.
Policies for all players are called \emph{policy profiles}.

A causal game is associated with a \emph{game graph} with square, round and diamond nodes for decision, chance and utility variables, respectively, with colours associating decision and utility nodes with different agents (\cref{def:cid}).
Edges into chance and utility nodes mirror those of an SCM, while
edges into decision nodes represent what information is available, i.e.\ $\evaralt \rightarrow \decisionvar$ is present if the outcome of $\evaralt$ is available when making the decision $\decisionvar$, with information edges displayed with dotted lines.
An example of a game graph is shown in \cref{subfig:mouse-mascid-real}.

Given a causal game, each agent can set a decision rule, ${\decisionrule}^D$, for each of their decisions, $\decisionvar$, which maps the information available at that decision to an outcome of the decision.
Formally, the decision rule $\pi^D$ is a deterministic function of $\Pa^D$ and $\exovar^D$, where $\exovar^D$ provides randomness to enable stochastic decisions.
This means that the decision rules can be combined with the causal game to form an SCM, which can be used to compute each agent's expected utility.
In the single agent case, the decision problem represented by the causal game is to select an optimal decision rule to maximise the expected utility (\cref{def:optimality}).
With multiple agents, solution concepts such as Nash Equilibrium (\cref{def:nash}) or Subgame-Perfect Nash Equilibrium (\cref{def:spe}) are needed, because in order to optimise their decision rules agents must also consider how other agents will will optimise theirs.

Similar to SCMs, interventions in a causal game can be realised as replacements of a subset of structural equations. 
However, in contrast to an SCM, an intervention can be made \emph{before} or \emph{after} decision rules are selected. 
This motivates a distinction between \emph{pre-policy} and \emph{post-policy} interventions \citep{hammond2021reasoning}.
Pre-policy intervention are made before the policies are selected, and agents may adapt their policies (according to some rationality principle) to account for the intervention. 
In other words, agents are made aware of the intervention before selecting their policies.
For post-policy interventions, the intervention is applied after the agents select their policies.
The agents cannot adapt their policies, even if their selected policies are no longer rational under the intervention. 
In other words, the intervention is applied without the awareness of the agents.

\section{Algorithms for Discovering Agents}
\label{sec:alg-disc-agents}

Having discussed some background material, we now begin our main contribution: providing algorithms to discover agents from causal experiments.

This can provide guidance on whether a proposed game graph is an accurate description of a system of agents and gives researchers tools for building game graphs using experimental data.

We propose three algorithms:
\begin{itemize}
    \item \cref{alg:loo-cd-xscm}, \emph{Mechanised Causal Graph Discovery},  produces an edge-labelled mechanised causal graph based on interventional data.
    \item \cref{alg:xcg-to-macid}, \emph{Agency Identification}, takes an edge-labelled mechanised causal graph and produces the corresponding game graph.
    \item \cref{alg:macid-to-xcg}, \emph{Mechanism Identification}, takes a game graph and draws the corresponding edge-labelled mechanised causal graph.
\end{itemize}\Cref{thm:a_3-a_2-xcg,thm:a_2-a_3-mascid,thm:correctness-algo} establish their correctness, and \cref{fig:overview} visualises their relationships.

\begin{figure}
\begin{minipage}[c]{0.7\textwidth}
    \centering
    \begin{tikzpicture}[ampersand replacement=\&]
        \matrix (mat) [table]
        {
        Game Theory \& Causality \\
        $\mathcal{I}$
        \&  \\
         \& $\xscg$ \\
        $\mascid$ \&  \\
        $\mascid$ \&  \\
         \& $\xscg$ \\
        $\mascid'$ \&  \\
         \& $\xscg$ \\
        $\mascid$ \&  \\
        \& $\xscg'$ \\
        };
        
        \draw 
            ([xshift=-.5\pgflinewidth]mat-1-1.south west) --   
            ([xshift=-.5\pgflinewidth]mat-1-2.south east);
        \draw 
            ([xshift=-.5\pgflinewidth]mat-4-1.south west) --   
            ([xshift=-.5\pgflinewidth]mat-4-2.south east);
        \draw 
            ([xshift=-.5\pgflinewidth]mat-7-1.south west) --   
            ([xshift=-.5\pgflinewidth]mat-7-2.south east);

        \begin{scope}[shorten >=7pt,shorten <= 7pt]
        \draw[->]  (mat-2-1.center) -- (mat-3-2.center) node[midway,sloped,above] {\cref{alg:loo-cd-xscm}};
        \draw[->]  (mat-3-2.center) -- (mat-4-1.center) node[midway,sloped,above] {\cref{alg:xcg-to-macid}};
        \draw[dashed, <->, red]  (mat-2-1.center) -- (mat-4-1.center) node[midway,left] {\cref{thm:correctness-algo}};
        
        \draw[->]  (mat-5-1.center) -- (mat-6-2.center) node[midway,sloped,above] {\cref{alg:macid-to-xcg}};
        \draw[->]  (mat-6-2.center) -- (mat-7-1.center) node[midway,sloped,above] {\cref{alg:xcg-to-macid}};
        \draw[dashed, <->, red]  (mat-5-1.center) -- (mat-7-1.center) node[midway,left] {\cref{thm:a_2-a_3-mascid}};
        
        \draw[->]  (mat-8-2.center) -- (mat-9-1.center) node[midway,sloped,above] {\cref{alg:xcg-to-macid}};
        \draw[->]  (mat-9-1.center) -- (mat-10-2.center) node[midway,sloped,above] {\cref{alg:macid-to-xcg}};
        \draw[dashed, <->, red]  (mat-8-2.center) -- (mat-10-2.center) node[midway,right] {\cref{thm:a_3-a_2-xcg}};
        \end{scope}
    \end{tikzpicture}
    \end{minipage}
    \begin{minipage}[c]{0.27\textwidth}
    \caption{Overview of our three theorems. Each provides relations between a game-theoretic mechanised causal game, $\mec{\mascim}$, with its interventional distributions, $\mathcal{I}$, and with its associated game graph, $\mascid$, and a causal object -- a mechanised causal graph, $\xscg$. Our proposed algorithms 
    \cref{alg:loo-cd-xscm}, \emph{Mechanised Causal Graph Discovery};
    \cref{alg:xcg-to-macid}, \emph{Agency Identification};
    and
    \cref{alg:macid-to-xcg}, \emph{Mechanism Identification}; detail how to transform from one representation to another. 
    }
    \label{fig:overview}
    \end{minipage}
\end{figure}

\subsection{Mechanised Structural Causal Model}
\label{sec:mechanised-scms}

In this subsection we introduce \emph{mechanised SCMs}, that we will later use in a procedure for discovering agents from experimental data.
A mechanised SCM is similar to an ordinary SCM, but includes a distinction between two types of variables: object-level and mechanism variables.
The intended interpretation is that the mechanism variables parameterise how the object-level variables depend on their object-level parents.
Mechanism variables have been called \emph{regime indicators} \citep{correa2020calculus} and \emph{parameter variables} \citep{dawid2002}.
Mechanised SCMs are variants of \emph{mechanised causal games} \citet{hammond2021reasoning} that lack explicitly labelled decision and utility nodes.
\Cref{subfig:mouse-xcg-model} draws the induced graph of a mechanised SCM.

\begin{definition}[Mechanised SCM]
\label{def:mechanised-scm}
    A \emph{mechanised SCM} is an SCM 
    in which there is a partition of the endogenous variables ${\xevars} = \evars \cup \mecvars{\evars}$ into object-level variables, $\evars$ (white nodes), and mechanism variables, $\mecvars{\evars}$ (black nodes),
    with $|\evars| = |\mecvars{\evars}|$.
    Each object-level variable $\evar$ has exactly one mechanism parent, denoted $\mecvar{\evar}$, that specifies the relationship between $\evar$ and the object-level parents of $\evar$.
\end{definition}

We refer to edges between object-level nodes as \emph{object-level edges} $E^{\textrm{obj}}$, edges between mechanism nodes as \emph{mechanism edges} $E^{\textrm{mech}}$, and edges between a mechanism node and the object-level node it controls \emph{functional edges} $E^{\textrm{func}}$.
We only consider mechanised SCMs in which the object-level-only subgraph is acyclic, but we allow cycles in the mechanism-only subgraph (we follow the formalism of \citet{bongers2016foundations} when using cyclic models).

By connecting mechanism variables with causal links, we violate the commonly taken \emph{independent causal mechanism assumption} \citep{scholkopf2021toward}, though we introduce a weaker form of it in \cref{assumption:non-dec-term-parentless} (see further discussion in \cref{sec:causal-discovery-review}).

Interventions in a mechanised SCM are defined in the same way as in a standard SCM, via replacement of structural equations.
An intervention on an object-level variable $V$ changes the value of $V$ without changing its mechanism, $\mec V$\footnote{Alternatively, it can be viewed as a path-specific intervention on $\mec V$ whose effects are constrained to $V$, and does not affect other mechanism variables (assuming that the domain of $\mec V$ is rich enough to facilitate the value $V$ is intervened to).}.
This can be interpreted as the intervention occurring after all mechanisms variables have been determined/sampled.

In a causal model, it is necessary to assume that the procedure for measuring and setting (intervening on) a variable is specified. 
Mechanised SCMs thereby assume a well-specified procedure for measuring and setting both object-level and mechanism variables.
Pre- and post-policy interventions in games correspond to mechanism and object-level interventions in mechanised SCMs \citep{hammond2021reasoning}.

The distinction between mechanism and object-level variables can be made more concrete by considering repeated interactions.
In \cref{sec:intro-example}, assume that the mouse is repeatedly placed in the gridworld, and can adapt its decision rule based (only) on previous episodes.
A mechanism intervention would correspond to a (soft) intervention that takes place across all time steps, so that the mouse is able to adapt to it.
Similarly, the outcome of a mechanism can then be measured by observing a large number of outcomes of the game, after any learning dynamics has converged%
\footnote{Strictly, for complete precision one would need an infinite number of games.}.
Finally, object-level interventions correspond to intervening on variables in one particular (post-convergence) episode.
Assuming the mouse is only able to adapt its behaviour based on previous episodes, it will have no way to adapt to such interventions.
\cref{sec:app-marg-merge} has a more detailed example of marginalising and merging nodes in a repeated game to derive the mechanised causal graph and game graph.

\subsection{Edge-labelled mechanised causal graphs}
\label{sec:edge-labelled}

We now introduce an edge-labelling on mechanised SCMs, aiming to capture two aspects of mechanised SCMs that we think are inherent to agents:
\begin{enumerate}
    \item whether a variable is inherently valuable to an agent (i.e.\ is a utility node), rather than just instrumentally valuable for something downstream;
    \item whether a variable's distribution adaptively responds for a downstream reason, (i.e.\ is a decision node), rather than for no downstream consequence (e.g.\ it's distribution is set mechanistically by some natural process). 
\end{enumerate}

For the first, to determine whether a variable, $W$, is inherently valuable to an agent, we can test whether the agent still changes its policy in response to a change in the mechanism for $W$ if the children of $W$ stop responding to $W$.
For the second, to determine whether a variable, $V$, adapts for a downstream reason, we can test whether $V$'s mechanism still responds even when the children of $V$ stop responding to $V$ (i.e.\ $V$ has no downstream effect).

We can stop the children of a variable responding to it by performing hard interventions on each child.
If an agent is present, we want it to be aware of these interventions, so they should be implemented via mechanism interventions -- we call this a structural mechanism intervention:

\begin{definition}[Structural mechanism intervention]
    A \emph{structural mechanism intervention} on a variable $V$ is an intervention $\mecvar{v}$ on its mechanism variable $\mecvar{V}$ such that $V$ is conditionally independent of its object-level parents.
    That is, under $\doo(\mecvar{V} = \mecvar{v} )$, the following holds
    \begin{align}
        \label{eq:structurlal-mech-int}
        \Pr(\evar \mid \Pa^\evar, \doo(\mecvar{V} = \mecvar{v} )) = \Pr(\evar\mid \doo(\mecvar{V} = \mecvar{v} )) .
    \end{align}
\end{definition}

We can record whether points 1. and 2. above hold in a label on the relevant mechanism edge motivating the following definition: 
\begin{definition}
\label{def:edge-labelled-scm}
    A mechanised SCM is \emph{edge-labelled} if it further identifies a subset $E^{\textrm{term}}\subseteq E^{\textrm{mech}}$ of mechanism edges (dashdotted blue) $\mecvar{W}\to \mecvar{V}$, called \emph{terminal} mechanism edges, such that:
    \begin{enumerate}
        \item $\mecvar{V}$ responds to $\mecvar{W}$ even after any effects of $W$ on its children, $\Ch^W$, have been removed by means of any structural mechanism interventions on ${\Ch^W}$; and
        \item $\mecvar{V}$ does not respond to $\mecvar{W}$ if effects of $V$ on its children, $\Ch^V$, have been removed by means of all structural mechanism interventions on ${\Ch^V}$.
    \end{enumerate}
    Non-terminal mechanism edges are drawn with dashed black lines.
\end{definition}

Intuitively, the terminal edges designate the variables that an agent cares about for their own sake.
For example, the mechanism edge $\mecvar{U}\to\mecvar{D}$ in \cref{subfig:mouse-xcg-model} is terminal, because it remains when the children of the object-level variable $U$ are cut (indeed, $U$ has no children), and disappears if we cut $D$ off from its children (since then $D$ doesn't affect $X$, and hence doesn't affect $U$).
In contrast, $\mecvar{X}\to\mecvar{D}$ is non-terminal, because if the object-level link $X\to U$ is cut (i.e., the agent's position is made independent of it finding cheese), then the agent will cease adapting its policy to changes in the slip probability $p$.
The labelling of terminal links will be used in \cref{sec:discover-game-graphs} to determine that $X$ is only instrumentally valuable to the agent.

\subsection{Discovering Edge-labelled, Mechanised Causal Graphs}
\label{sec:discovering-edge-labelled-mcgs}

We next describe how edge-labelled, mechanised causal graphs can be inferred from interventional data.
Intuitively, by definition of a causal edge,
if one applies interventions to all nodes except one node $V$, and varying these interventions at only node $W$, then one can reliably discover whether there should be a causal edge from $W$ to $V$ (even in cyclic SCMs).
This \emph{leave-one-out} strategy%
\footnote{\emph{Leave-one-out} is not necessarily the most efficient procedure. Since we restrict to an acylic object-level subgraph, a more efficient standard (acyclic) causal discovery algorithm could replace the leave-one-out strategy described here (see, e.g., \citealp{eberhardt2012number}). For the mechanism subgraph, a more efficient cyclic causal discovery algorithm could be used, (e.g., \citealp{forre2018constraint}). There are usually tradeoffs between speed and assumptions required by these algorithms, however.}
is described below:

\begin{algorithm}
    \caption*{\emph{Leave-one-out causal discovery}
    }\label{alg:loo}
    \begin{algorithmic}[1]
    \Input 
    Interventional distributions $\mathcal{I} =\{P(\evars \mid {\doo(\bm{Y}=\bm{y}}))\}$ over variables $\bm{V}$
    \State $ E \gets \varnothing$
    \For{$\evar \in {\evars}$}\label{alg:obj-loop-start}
        \For{$\evaralt \in  \bm{V}\setminus \{\evar\}$}
        \State $\bm{Y} \gets \bm{V}\setminus\{\evar,\evaralt\}$
            \For{
            $\bm{y}\in\dom(\bm{Y})$ and $w, w' \in \dom(W)$
            }
                \If{
                $P(\evar \mid \doo(\bm{Y}=\bm{y},W=w)) \neq P(\evar \mid \doo(\bm{Y}=\bm{y},W=w'))$
                }
                    \State $E \gets E \cup (\evaralt,\evar)$
                    \State \textbf{break}
                \EndIf
            \EndFor
        \EndFor
    \EndFor\label{alg:obj-loop-end}
    \Output $(\bm{V}, E)$
    \end{algorithmic}
\end{algorithm}

\begin{lemma}[Leave-one-out causal discovery]
    Applied to the set of interventional distributions generated by a (potentially cyclic) SCM, \emph{Leave-one-out causal discovery} returns the correct causal graph.
\end{lemma}
\begin{proof}
    Immediate from the definitions of SCM and causal graph, see \cref{sec:background}.
\end{proof}

\Cref{alg:loo-cd-xscm} applies \emph{Leave-one-out causal discovery} to the combined set of object-level and mechanism variables of a mechanised SCM, and then infers edge-labels using structural mechanism interventions on object-level children.

\begin{algorithm}
    \caption{\emph{Edge-labelled mechanised SCM discovery}
    }\label{alg:loo-cd-xscm}
    \begin{algorithmic}[1]
    \Input 
    Interventional distributions $\mathcal{I}=\{P(\xevars \mid {\doo(\bm{Y}=\bm{y}}))\}$ over variables $\xevars = \evars \cup \mecvars{\evars}$
    \State $ (\xevars, E) \gets \text{leave-one-out-causal-discovery}(\{P(\xevars \mid {\doo(\bm{Y}=\bm{y}}))\})$
    \State $ E^{\textrm{obj}} \gets \{(W, V)\in E: W, V \in \evars\}$ 
    \State $ E^{\textrm{mech}} \gets \{(W, V)\in E: W, V \in \mecvars{\evars}\}$
    \State $ E^{\textrm{func}} \gets \{(W, V)\in E: V \in \evars, W\in\mecvar{\evars}\}$
    \If{$E\not= E^{\textrm{obj}}\cup E^{\textrm{mech}} \cup E^{\textrm{func}}$  or  $\exists V: |\{(W, V)\in E^{\textrm{func}}: W\in \mecvar{\evars}\}|\not=1$}
    \State Error: graph is not a mechanised SCM
    \EndIf
    \State $ E^{\textrm{term}} \gets \varnothing$
    \For{$(\mecvar{W}, \mecvar{V}) \in E^{\textrm{mech}}$}\label{alg:mech-loop-start}
        \State $\mecvars{Y} \gets \mecvar{\evars}\setminus\{\mecvar{W},\mecvar{V}\}$
        \For{
        interventions $\mecvars{y}\cup\mecvar{\ch^W}$ that are structural for $\Ch^W$, and interventions $\mecvar{w}, \mecvar{w'}$ on $\mecvar{W}$} \label{alg:util-check-start}
            \If{$
                P(\mecvar{\evar} \mid \doo(\mecvars{Y}=\mecvars{y},\mecvar{\Ch^W}=\mecvar{\ch^W},\mecvar{W}=\mecvar{w})) \neq P(\mecvar{\evar} \mid \doo(\mecvars{Y}=\mecvars{y},\mecvar{\Ch^W}=\mecvar{\ch^W},\mecvar{W}=\mecvar{w'}))
            $ } \label{alg:obj-change-measured}
                \State $E^{\textrm{term}} \gets E^{\textrm{term}} \cup (\mecvar{W}, \mecvar{V})$
                \State \textbf{break}
                \label{alg:util-check-end}
            \EndIf
        \EndFor
        \For{
        interventions $\mecvars{y}\cup\mecvar{\ch^V}$ that are structural for $\Ch^V$, and interventions $\mecvar{w}, \mecvar{w'}$ on $\mecvar{W}$}  \label{alg:dec-check-start}
            \If{$
                P(\mecvar{\evar} \mid \doo(\mecvars{Y}=\mecvars{y},\mecvar{\Ch^V}=\mecvar{\ch^V},\mecvar{W}=\mecvar{w})) \neq P(\mecvar{\evar} \mid \doo(\mecvars{Y}=\mecvars{y},\mecvar{\Ch^V}=\mecvar{\ch^V},\mecvar{W}=\mecvar{w'}))
            $ } \label{alg:dec-change-measured}
                \State $E^{\textrm{term}} \gets E^{\textrm{term}} \setminus (\mecvar{W}, \mecvar{V})$
                \State \textbf{break}
                \label{alg:dec-check-end}
            \EndIf
        \EndFor
    \EndFor\label{alg:mech-loop-end}
    \Output $(\evars \cup \mecvars{\evars},\, E^{\textrm{obj}} \cup E^{\textrm{mech}}\cup E^{\textrm{func}},\, E^{\textrm{term}})$
    \end{algorithmic}
\end{algorithm}

\begin{lemma}[Discovery of mechanised SCM]
    Applied to the set of interventional distributions generated by a mechanised SCM in which structural mechanism interventions are available for all nodes, \cref{alg:loo-cd-xscm} returns the correct edge-labelled mechanised causal graph.
\end{lemma}

\begin{proof}
    The algorithm checks the conditions in \cref{def:mechanised-scm,def:edge-labelled-scm}.
\end{proof}

Applied to the mouse example of \cref{fig:mouse}, \cref{alg:loo-cd-xscm} would take interventional data from the system and draw the edge-labelled mechanised causal graph in \cref{subfig:mouse-xcg-model}. 
For example, the edge $(\mecvar{\utilvar}, \mecvar{\decisionvar})$ will be discovered because the mouse's decision rule will change in response to a change in the distribution for cheese location.

\subsection{Discovering game graphs}
\label{sec:discover-game-graphs}

To discover agents,we can convert an edge-labelled mechanised causal graph into a game graph as specified by  \cref{alg:xcg-to-macid}:
decision nodes are identified by their mechanisms having incoming terminal edges (Line~\ref{alg:dec-set}), while utility nodes are identified by their mechanisms having outgoing terminal edges (Line~\ref{alg:util-set}).
Decisions and utilities that are in the same connected component in the terminal edge graph receive the same colouring, which is distinct from the other components. On Line~\ref{alg:connected-cololur}, $\textrm{Connected}_{term}(V)$ is the set $\{ W \in \evars | \mecvar{W} - - \ \mecvar{V} \}$ where the undirected path $\mecvar{W} - - \  \mecvar{V}$ is in the terminal edge graph. This set could be found by a search algorithm, such as breadth-first search.

\begin{algorithm}
\caption{\emph{Agency Identification}. Converts edge-labelled mechanised causal graph to game graph}\label{alg:xcg-to-macid}
\begin{algorithmic}[1]
\Input An edge-labelled mechanised causal graph $\scg= (\xevars, E)$, with nodes $\xevars = \evars \cup \mecvars{\evars}$ and edges $ E = E^{\textrm{obj}} \cup E^{\textrm{func}}\cup E^{\textrm{mech}}$, with $E^{\textrm{term}} \subseteq E^{\textrm{mech}} $.
\State $\decisionvars \gets \varnothing$ 
\State $\utilvars \gets \varnothing$ 
\For{$(\mecvar{W}, \mecvar{\evar}) \in E^{\textrm{term}}$}\label{alg:decsion-id-loop-start} 
        \State $\decisionvars \gets \decisionvars \cup \{{\evar}\}$ \label{alg:dec-set}
        \State $\utilvars\gets \utilvars\cup \{W\}$ \label{alg:util-set}
\EndFor  
\State $\structvars \gets \evars \setminus \{\utilvars  \cup \decisionvars\}$
\State  $\textrm{Coloured} \gets \varnothing$
\For{$\evar \in \utilvars \cup \decisionvars$}
    \If{$\evar \not \in \textrm{Coloured}$}
        \State $\textrm{Connected}_{term}(V) \gets \textrm{NewColour}$ \label{alg:connected-cololur}
        \State $\textrm{Coloured} \gets \textrm{Coloured} \cup \textrm{Connected}_{term}(V)$
    \EndIf
\EndFor 
\Output game graph $\mascid = (N, \evars, E)$, where
$\evars= \structvars \cup \utilvars  \cup \decisionvars$ and $E=E^{\textrm{obj}}$.
\end{algorithmic}
\end{algorithm}

\Cref{alg:xcg-to-macid} will identify any decision node $D$ under the following conditions (though it may generate false positives):
\begin{itemize}
    \item A utility node $U$, or a mediator node $X$ on a directed path from $D$ to $U$, is included in $\evars$.
    \item
    The utility/mediator node must be sufficiently important to the agent controlling $D$ that its mechanism shapes the agents behaviour.
    \item Mechanism interventions are available that change the agent's optimal policy for controlling $U$ (or $X$).
    \item These mechanism interventions are operationalised in a way that the agent's policy can respond to the changes they imply.
\end{itemize}

Under the following stronger assumptions\footnote{We consider examples of breaking the first of these assumptions in \cref{sec:breaking-assumptions}.}, \cref{alg:xcg-to-macid} is guaranteed to produce a fully correct game graph (without false positives).
These assumptions are most easily stated using mechanised SCMs with labelled decision and utility nodes.
Following \citet{hammond2021reasoning}, we call such objects \emph{mechanised games}.
%We prove the correctness of this algorithm under the following assumptions.

For our first assumption, the following definition will be helpful.
\begin{definition}
\label{def:agent-subgraph}
    For a game graph, $\macid$, we define the \emph{agent subgraph} to be the graph $\macid^A = (\decisionvars^A \cup \utilvars^A, E^A)$, where the edge $(D, U)$ belongs to $E^A$ if and only if there is a directed path $D \pathto U \in \macid$ that doesn't pass through any $U' \in \utilvars^A \setminus \{U\}$.
    We define the \emph{decision-utility subgraph} to be the graph $\macid^{\decisionvars\utilvars} = (\decisionvars \cup \utilvars, \cup_A E^A)$.
\end{definition}

For example, the decision-utility subgraph of \cref{subfig:mouse-mascid-real} consists of two nodes, $D$ and $U$, and an edge $(D, U)$ as there is a directed path $D$ to $U$ that is not mediated by other utility nodes.
One further piece of terminology we use is that a DAG is weakly connected if replacing all of its directed edges with undirected edges produces a connected graph, i.e. one in which every pair of vertices is connected by some path. 
A weakly connected component is a maximal subgraph such that all nodes are weakly connected.
For example, the decision-utility subgraph of \cref{subfig:mouse-mascid-real} is connected, and consists of a single connected component (the agent subgraph for the mouse).

Our first assumption uses these definitions as follows:
\begin{assumption}
    \label{assumption:ea-components-agents}
    Each weakly connected component of the decision-utility subgraph is an agent subgraph, and contains at least one decision and one utility node.
\end{assumption}
The intuition behind this assumption is that if there was a disconnected component in the agent subgraph, then the decisions in that component could be reasoned about independently from the rest of the decisions, and there would be no way to experimentally distinguish whether those independent decisions were made by a separate agent. So we make this as a simplifying assumption that only separate agents reason about their decisions independently.
An example of a game ruled out by this assumption is \cref{fig:ndu}, in which a decision doesn't directly cause it's utility.

\begin{assumption}\label{assumption:opt-decision-rule}
   For any set of mechanism interventions, every agent optimises expected utility (plays best response) in every decision context, i.e.\ agents play a subgame perfect equilibrium.
\end{assumption}

\cref{assumption:opt-decision-rule} implies that mechanism interventions are operationalised in a way that agents can appropriately respond to them, that agents are trying to optimise their utility nodes, and that object-level links going into the decision can be interpreted as information links (since agents adapt appropriately to the outcomes of the decision parents).

\begin{assumption}
\label{assumption:preferred-order}
Agents have a preferred ordering over decision rules, so that if two or more decision rules obtain the same (optimal) expected utility in all decision contexts, the agent will always pick the first optimal decision rule according to the order.
\end{assumption}

This ensures no unmotivated switches occur -- so that agents don't switch decision rule in response to mechanism interventions which have no effect on the optimality of that decision rule.

\begin{assumption}
\label{assumption:non-dec-term-parentless}
   Only decision nodes, $D \in \decisionvars$, have mechanisms, $\mecvar{D}$, with ingoing terminal edges.
\end{assumption}

This is a weak form of the popular \emph{independent causal mechanism} assumption \citep{scholkopf2012causal}, discussed further in \cref{sec:causal-discovery-review}, preventing dependencies between certain mechanisms.

\begin{assumption}
\label{assumption:mechanism-richness}
For each node $V$, interventions on $\mecvar{V}$ can instantiate any deterministic function relating $V$ to its parents (when $V$ lacks parents, it can be set to any constant value in $\dom(V)$).
\end{assumption}

This is to ensure that we can enact the necessary soft interventions, in a way that the agent is aware of. We are now ready to establish a correctness result for \cref{alg:xcg-to-macid}.

\begin{theorem}[Correctness of \cref{alg:loo-cd-xscm,alg:xcg-to-macid}]
\label{thm:correctness-algo}
    Let $\mec{\mascim}_{\textrm{real}}$ be a mechanised causal game satisfying \cref{assumption:ea-components-agents,assumption:opt-decision-rule,assumption:preferred-order,assumption:non-dec-term-parentless,assumption:mechanism-richness}. 
    Let  $\macid_{\textrm{model}}$ be the game graph resulting from applying \cref{alg:loo-cd-xscm} followed by \cref{alg:xcg-to-macid} to $\mec{\mascim}_{\textrm{real}}$. Then $\macid_{\textrm{model}} = \macid_{\textrm{real}}$.
\end{theorem}

\begin{proof}
    We establish that the algorithm infers the correct object-level causal structure, the correct labelling of decision and utility nodes (and hence of chance nodes), and the correct colouring of the same.
    
    {\bf{Causal structure}}
    The only structural difference between a game and an SCM is the presence of information links in the game. By \cref{assumption:mechanism-richness}, we can impute an arbitrary decision rule to any decision, that makes it depend on all its observations. Thereby all information links are causal links.
    
    {\bf{Decision:}} 
    We first show that all and only decisions get mapped to decisions.
    Let $\decisionvar \in \decisionvars^A$ be a decision variable for agent $A$ in $\mec{\mascim}_{\textrm{real}}$. 
    By \cref{assumption:ea-components-agents} we have that there exists a utility variable $\utilvar \in \utilvars^A$ such that there's a directed path, $p$, from $D$  to $U$ not passing through any other utility node of $A$.
    % which is a descendent of $\decisionvar$, i.e. $\utilvar \in \Desc^\decisionvar$,
    % \zac{could relax assumption to just that there is a utility for every decision (need not be a descendent for this sentence)}
    %  and where $A$ has no other utility node on a particular directed path $p$ from $D$ to $U$.
    By means of mechanism interventions, we can ensure that $U$ is either 0 or 1 depending on the value of $D$ by copying the value of $D$ along $p$, using deterministic functions (\cref{assumption:mechanism-richness}).
    All other nodes ignore $D$.
    Agent $A$ chooses a decision rule setting $U$ to 1 (\cref{assumption:opt-decision-rule}).
    If we do a mechanism intervention to invert the function governing $U$, and cut off all of its effects on its children, then agent $A$ will choose a different decision rule and Lines~\ref{alg:util-check-start}-\ref{alg:util-check-end} will add edge $(\mecvar{U}, \mecvar{D})$ to $E^{\textrm{term}}$.
    Further, no mechanism intervention on the function governing $U$ will cause agent $A$ to choose a different decision rule if we intervene to cut the effect of $D$ on its children as all decision rules would have the same expected utility (and \cref{assumption:preferred-order} rules out unmotivated switches). Thus, Lines~\ref{alg:dec-check-start}-\ref{alg:dec-check-end} will not remove $(\mecvar{U}, \mecvar{D})$ from $E^{\textrm{term}}$. \cref{alg:xcg-to-macid} then correctly identifies $D$ as a decision.
    
    Conversely, assume $\evar\in \evars\setminus \decisionvars$ is a non-decision. It may be that Lines~\ref{alg:util-check-start}-\ref{alg:util-check-end} will add $(\mecvar{W}, \mecvar{V})$ to $E^{\textrm{term}}$, for some $W \in \evars\setminus\{V\}$. But Lines~\ref{alg:dec-check-start}-\ref{alg:dec-check-end} will remove $(\mecvar{W}, \mecvar{V})$ from $E^{\textrm{term}}$ by \cref{assumption:non-dec-term-parentless}, and \cref{alg:xcg-to-macid} then doesn't identify $V$ as a decision.
    
    {\bf{Utility:}} 
    We next show that all and only utilities get mapped to utilities.
    Let $\utilvar \in \utilvars^A$ be a utility variable for agent $A$ in $\mec{\mascim}_{\textrm{real}}$. By \cref{assumption:ea-components-agents} we have that there exists a decision variable $\decisionvar \in \decisionvars^A$ such that there's a directed path, $p$, from $D$  to $U$ not passing through any other utility node of $A$.
    By the same construction as for decision nodes above, \cref{alg:loo-cd-xscm} will discover a terminal mechanism edge $(\mecvar{U}, \mecvar{D})$.
    Therefore \cref{alg:xcg-to-macid} identifies $U$ to be a utility as desired.
    
    Conversely, consider a non-utility node, $W\not \in \utilvars$, and some other node, $V \in \evars\setminus\{W\}$, with structural interventions cutting off $\Ch^W$ and interventions on all mechanisms except $\mecvar{V}$. Suppose, for contradiction, there exists a terminal edge $(\mecvar{W}, \mecvar{V})$. By \cref{assumption:non-dec-term-parentless}, there will be a terminal edge $(\mecvar{W}, \mecvar{V})$ only if $V$ is a decision. Further, by \cref{assumption:preferred-order,assumption:opt-decision-rule} the expected utility must be affected by the change in $\mecvar{W}$. But since we've intervened on all mechanisms except $\mecvar{V}$, the only effect  $\mecvar{W}$ can have on the expected utility is via $W$. But $W \not \in \utilvars$, and $\Ch^W$ are not affected (since they've been cut off), so $\mecvar{W}$ cannot affect expected utility.
    Therefore, only utility variables get outgoing edges in $E^{\textrm{term}}$ from \cref{alg:loo-cd-xscm}, and \cref{alg:xcg-to-macid} does not assign $\evar$ to be a utility.
    
    We have thus shown that all and only decisions nodes get mapped to decisions, and similarly for utilities. All that are left are chance nodes, and these must be mapped to chance nodes (since only decisions/utilities get mapped to decisions/utilities).
    
    {\bf{Colouring:}} 
    By \cref{assumption:ea-components-agents} for any agent, $A$, and for any decision $D \in D^A$, there exists $U \in U^A$ with $(D,U) \in E^A$. By the above paragraphs, we must have that $E^{\textrm{term}}$ contains the edge $(\mecvar{U}, \mecvar{D})$, and further, by the converse arguments, the only edges in $E^{\textrm{term}}$ are of the form $(\mecvar{U}, \mecvar{D})$ with $D \in D^A, U \in U^A$ and $(D,U) \in E^A$ for some $A$, which means $E^{\textrm{term}}$ is a disjoint union of $\mecvar{E}^A$, in which each edge of $\mecvar{E}^A$ is the reverse of an edge in $E^A$. By \cref{assumption:ea-components-agents}, the weakly connected components of $G^{\decisionvars\utilvars}$ are the $G^A$, and so the $\mecvar{E}^A$ are each weakly connected, and disconnected from each other. The colouring of \cref{alg:xcg-to-macid} colours each vertex of a connected component the same colour, and distinctly to all other components, and thus is correct. 
\end{proof}

\subsection{Mechanism Identification Procedure}
\label{sec:mech-id}
In the last section we demonstrated an algorithm that, when applied after a causal discovery algorithm, can identify the underlying game graph of a system.
In this section we will show the converse, that if one already has a game graph, one can convert it into an edge-labelled mechanised causal graph. The interpretation is that the same underlying system can equivalently be represented either as an edge-labelled mechanised causal graph, which is a physical representation of the system, or as a game graph, which is a decision-theoretic representation of the system.

We first prove a Lemma relating the mechanism causal graph produced by \cref{alg:loo-cd-xscm} to \emph{strategic relevance} \citep{Koller2003-yh}, which captures which other decision rules are relevant for optimising the decision rule at $\decisionvar$. \citeauthor{Koller2003-yh} give a sound and complete graphical criterion for strategic relevance, called \emph{s-reachability}\footnote{Our definition here generalises the definition from \cite{Koller2003-yh} to include non-decision variables as being s-reachable, following \cite{hammond2021equilibrium}.}, where $\evar \neq \decisionvar$ is \emph{s-reachable} from $\decisionvar \in \decisionvars^A$, for agent $A$, if, in a modified game graph $\hat{\macid}$ with a new parent $\hat{\evar}$ added to $\evar$, we have $\hat{\evar} \not \perp_{\hat{\mascid}} \utilvars^\decisionvar \mid \Fa^\decisionvar $, where $\utilvars^\decisionvar$ is the set of utilities for agent $A$ that are descendants of $\decisionvar$ (i.e. $\utilvars^\decisionvar = \utilvars^A\cap \Desc^\decisionvar$ for $\decisionvar \in \decisionvars^A$) and $\not \perp$ denotes d-connection \citep{pearl2009causality}.
In the game graph in \cref{subfig:mouse-mascid-real}, both $X$ and $U$ are s-reachable from $\decisionvar$.

\begin{lemma}
\label{lem:s-reach-implies-mech-parent}
    Let $\mec{\mascim}$ be a mechanised causal game satisfying \cref{assumption:ea-components-agents,assumption:opt-decision-rule,assumption:preferred-order,assumption:non-dec-term-parentless,assumption:mechanism-richness}, containing an agent, $A$, with decision variables $\decisionvars^A$ and utility variables $\utilvars^A$, and let $\xscg$ be the mechanised causal graph with edges $E^{\textrm{obj}} \cup E^{\textrm{func}} \cup E^{\textrm{mech}}$, and $E^{\textrm{term}} \subseteq E^{\textrm{mech}}$, which results from applying \cref{alg:loo-cd-xscm} to $\mec{\mascim}$.
    Then
    \begin{enumerate}
        \item For $\decisionvar \in \decisionvars^A$, that the node $Y \in \evars\setminus \decisionvar$ is s-reachable from $D$ is a necessary and sufficient condition for $(\mecvar{Y},{\mecvar{\decisionvar}}) \in E^{\textrm{mech}}$ (this places no restriction on $(\mecvar{Y},{\mecvar{W}}) \in E^{\textrm{mech}}$ for $W \not\in\decisionvars$).
        \item Further, for $Y\in \utilvars^A$, that the existence of a directed path $D \pathto Y$ not through another $U' \in \utilvars^A\setminus \{Y\}$ is a necessary and sufficient condition for $(\mecvar{Y},{\mecvar{\decisionvar}}) \in E^{\textrm{term}}$.
    \end{enumerate}
\end{lemma}

\begin{proof}

    Necessity of 1:
    We largely follow the soundness direction of \cite{Koller2003-yh}, Thm 5.1, with an extension to relate this to a mechanised causal graph discovered by \cref{alg:loo-cd-xscm}. The proof strategy is to suppose that $Y$ is \emph{not} s-reachable from $\decisionvar$, and show this implies $(\mecvar{Y}, {\mecvar{\decisionvar}}) \not\in E^{\textrm{mech}}$. 
    
    We perform the mechanism interventions, $\doo(\mecvars{\bm{\evaralt}}=\mecvars{w}, \mecvar{Y}=\mecvar{y})$ and $\doo(\mecvars{\bm{\evaralt}}=\mecvars{w}, \mecvar{Y}=\mecvar{y}')$.
    Since $\decisionvar$ is a decision variable, by Lemma 5.1 of \cite{Koller2003-yh} the optimal decision rule $\decisionrule^\decisionvar_{\mecvar{y}}(\pa^\decisionvar, \exovar^\decisionvar)$ under $\doo(\mecvars{\bm{\evaralt}}=\mecvars{w}, \mecvar{Y}=\mecvar{y})$ must be a solution of the following optimisation problem 
    \begin{align}
        \argmax_{\pi}\sum_{\decisionval \in \dom(\decisionvar)}\pi(d)
        \sum_{\utilval \in \dom(\utilvar^\decisionvar)}
        P(\utilval \mid \decisionval, \pa^\decisionvar, \doo(\mecvars{\bm{\evaralt}}=\mecvars{w}, \mecvar{Y}=\mecvar{y}))
        \cdot \utilval
    \end{align}
    and similarly for the decision rule $\decisionrule^\decisionvar_{\mecvar{y}'}$ 
    under $\doo(\mecvars{\bm{\evaralt}}=\mecvars{w}, \mecvar{Y}=\mecvar{y}')$.

    Now suppose that $Y$ is \emph{not} s-reachable from $\decisionvar$, then by Lemma 5.2 of \cite{Koller2003-yh}, we have that 
    $P(\utilval \mid \decisionval, \pa^\decisionvar, \doo(\mecvars{\bm{\evaralt}}=\mecvars{w}, \mecvar{Y}=\mecvar{y}))
    =
    P(\utilval \mid \decisionval, \pa^\decisionvar, \doo(\mecvars{\bm{\evaralt}}=\mecvars{w}, \mecvar{Y}=\mecvar{y}'))$, 
    and so the two optimization problems are the same. Since they are solutions of the same optimization problem, and by \cref{assumption:opt-decision-rule,assumption:preferred-order} the agents choose decision rules which make up subgame-perfect equilibrium, this leads to the same decision rule in each intervened game $\decisionrule^\decisionvar_{\mecvar{y}}(\pa^\decisionvar, \exovar^\decisionvar) = \decisionrule^\decisionvar_{\mecvar{y}'}(\pa^\decisionvar, \exovar^\decisionvar)$. 
    This holds for any $\mecvars{W},\mecvar{y},\mecvar{y}'$ and so \cref{alg:loo-cd-xscm} does not draw an edge, i.e. $(\mecvar{Y}, {\mecvar{\decisionvar}}) \not\in E^{\textrm{mech}}$, as was to be shown.
    
    Necessity of 2: 
    As argued in \cref{thm:correctness-algo} (colouring), $(\mecvar{Y},{\mecvar{\decisionvar}}) \in E^{\textrm{term}}$ implies $(D, Y) \in E^A$, which by \cref{def:agent-subgraph} means there exists a directed path $D \pathto Y$ not through another $U' \in \utilvars^A\setminus \{Y\}$.
  
    Sufficiency of 1: 
    We can use soft interventions on object-level variables to construct the same model as used in the existence proof for Theorem 5.2 of \cite{Koller2003-yh}. We note that the proof for Theorem 5.2 of \cite{Koller2003-yh} is written for another decision variable $\decisionvar'$ being s-reachable from $\decisionvar$. But the proof itself makes no use of the special nature of $\decisionvar'$ as a decision, rather than any other type of variable, and so it also applies to any variable $Y \in \evars\setminus\{\decisionvar\}$. 
    
    Suppose $Y$ is s-reachable from $\decisionvar$ in $\mec{\mascim}$.
    It follows from Theorem 5.2 of \cite{Koller2003-yh} that the optimal decision rule for $\decisionvar$ will be different under these mechanism interventions (i.e. this choice of causal game), when different mechanism interventions are applied to $Y$. Hence \cref{alg:loo-cd-xscm} will draw an edge $(\mecvar{Y}, {\mecvar{\decisionvar}}) \in E^{\textrm{mech}}$.

    Sufficiency of 2:
    By the arguments in \cref{thm:correctness-algo} (decision, utility) the existence of a directed path $D \pathto Y$ not through another $U' \in \utilvars^A\setminus \{Y\}$ means that $(\mecvar{Y},{\mecvar{\decisionvar}}) \in E^{\textrm{term}}$.

\end{proof}

\begin{algorithm}
\caption{\emph{Mechanism Identification}. Convert game graph to edge-labelled mechanised causal graph.}\label{alg:macid-to-xcg}
\begin{algorithmic}[1]
\Input game graph $\mascid = (N, \evars, E)$ 
\State $E^{\textrm{term}} \gets \varnothing$
\State $\mecvars{\evars} \gets \varnothing$
\For{$\evar \in \evars$}\label{alg:func-loop-start}
    \State ${E} \gets {E} \cup (\mecvar{\evar}, \evar)$
    \State $\mecvars{\evars} \gets \mecvars{\evars} \cup \textrm{Node}(\mecvar{\evar})$, 
\EndFor\label{alg:func-loop-end}
\State ${\xevars} \gets \evars \cup \mecvars{\evars}$, 

\For{$A \in N$} 
    \For{$\decisionvar \in \decisionvars^A$}\label{alg:sreach-loop-start}
        \For{$\evar \in \evars\setminus\{\decisionvar\}$}
            \State $\hat{\mascid}$ is ${\mascid}$ with a new parent $\hat{\evar}$ added to $\evar$
            \If{$\hat{\evar} \not \perp_{\hat{\mascid}} \utilvar^\decisionvar \mid \{\Pa^\decisionvar \cup \decisionvar\}$} \label{alg:macid-xcg-s-reach-start}
                \State ${E} \gets {E} \cup (\mecvar{\evar}, \mecvar{\decisionvar})$ \label{alg:macid-xcg-s-reach-end}
            \EndIf
            \If{$\exists$ directed path $D \pathto  \evar$ not through another $U'\in \utilvars^A\setminus\{\evar\} $} \label{alg:macid-xcg-term-start}
                \State $E^{\textrm{term}} \gets E^{\textrm{term}} \cup (\mecvar{\evar}, \mecvar{\decisionvar})$ \label{alg:macid-xcg-s-term-end}
            \EndIf
        \EndFor
    \EndFor
\EndFor\label{alg:sreach-loop-end}

\Output mechanised causal graph $\xscg= (\xevars, {E})$, $E^{\textrm{term}}$
\end{algorithmic}
\end{algorithm}

The conversion from game graph to mechanised causal graph is done by \cref{alg:macid-to-xcg}, \emph{Mechanism Identification}, which identifies mechanisms by converting a game graph into a mechanised causal graph. It first takes the game graph edges and on Lines~\ref{alg:func-loop-start}-\ref{alg:func-loop-end} adds the function edges. Lines~\ref{alg:sreach-loop-start}-\ref{alg:sreach-loop-end} then add the mechanism edges based on s-reachability: if a node $\evar$ is s-reachable from $\decisionvar$ in the game graph, then we include an edge $(\mecvar{\evar}, \mecvar{\decisionvar})$ in the mechanised causal graph.
Further, it adds a terminal edge when there's a directed path from one of an agent's decisions to one of its utilities, that doesn't pass through another of its utilities.
We now establish that \cref{alg:xcg-to-macid} and \cref{alg:macid-to-xcg} are inverse to each other. We will use the shorthand $a_i(x)$, for $i=1,2,3$ to refer to the result of algorithm $a_i$ on object $x$, where e.g. $x$ is a game graph.

\begin{theorem}[\cref{alg:xcg-to-macid} is a left inverse of \cref{alg:macid-to-xcg}]
\label{thm:a_2-a_3-mascid}
    Let $\mascid$ be a mechanised game graph satisfying \cref{assumption:ea-components-agents,assumption:opt-decision-rule,assumption:preferred-order,assumption:non-dec-term-parentless,assumption:mechanism-richness}, and let $\xscg$ be the mechanised causal graph resulting from applying \cref{alg:macid-to-xcg} to it. Then applying \cref{alg:xcg-to-macid} on $\xscg$ reproduces $\mascid$. That is, $a_2(a_3(\mascid)) = \mascid$.
\end{theorem}

\begin{proof}
   All edges between nodes are the same in $\mascid$ and $a_2(a_3(\mascid))$, because neither \cref{alg:xcg-to-macid}
   or \cref{alg:macid-to-xcg} changes the object-level edges. We will now show that the node types are the same in both. 
   
   {\bf{Decision:}}
   Let $A$ be an agent with utilities $\utilvars^A$ and let $\decisionvar \in \decisionvars^A$, then by \cref{assumption:ea-components-agents} $\exists \utilvar \in \utilvars^A$ and a directed path $D \pathto U$ not through another $U' \in \utilvars^A\setminus\{U\}$. 
   \cref{alg:macid-to-xcg} Lines~\ref{alg:macid-xcg-term-start}-~\ref{alg:macid-xcg-s-term-end} add $(\mecvar{\utilvar}, \mecvar{\decisionvar})$ to $E^{\textrm{term}}$. \cref{alg:xcg-to-macid} then adds $\decisionvar$ to the set of decisions, as desired.
   %\ramana{Line number for $E^{\textrm{term}}$ seems off}
   %\zac{fixed}

   Let $\evar \in \evars\setminus\decisionvars$. \cref{alg:macid-to-xcg} Lines~\ref{alg:macid-xcg-term-start}-~\ref{alg:macid-xcg-s-term-end} only adds terminal mechanism edges going into decisions, and \cref{alg:xcg-to-macid} then doesn't add $\evar$ to the set of decisions, as desired.

   {\bf{Utility:}} 
    Let $A$ be an agent with decisions $\decisionvars^A$ and let $\utilvar \in \utilvars^A$, then by \cref{assumption:ea-components-agents} $\exists \decisionvar \in \decisionvars^A$ and a directed path $D \pathto U$ not through another $U' \in \utilvars^A\setminus\{U\}$. So \cref{alg:macid-to-xcg} Lines~\ref{alg:macid-xcg-term-start}-\ref{alg:macid-xcg-s-term-end} add $(\mecvar{\utilvar}, \mecvar{\decisionvar})$ to $E^{\textrm{term}}$. \cref{alg:xcg-to-macid} then adds $\utilvar$ to the set of utilities, as desired.
    
     Let $\evar \in \evars\setminus\utilvars$. \cref{alg:macid-to-xcg} Lines~\ref{alg:macid-xcg-term-start}-\ref{alg:macid-xcg-s-term-end} only adds terminal edges going out of utilities, so there will be no edge out of $\mecvar{\evar}$ in $E^{\textrm{term}}$. \cref{alg:xcg-to-macid} then doesn't add $\evar$ to the set of utilities, as desired.

    {\bf{Colouring:}} 
    By above paragraphs, the node types and edges are the same in both $a_2(a_3(\mascid))$ and $\mascid$. 
    By \cref{assumption:ea-components-agents} the colouring in $\mascid$ is a property of the connectedness and hence will be the same in $a_2(a_3(\mascid))$.
\end{proof}

We now consider the other direction: beginning with a mechanised causal graph, can we transform it to a game graph and then back to the same mechanised causal graph? In general this isn't possible, because the space of possible mechanised causal graphs is larger than the space of mechanised causal graphs that can be recovered using only the information present in a game graph. In particular, mechanisms with non-terminal incoming mechanism edges do not, in general, get codified in the game graph when using $a_2$.
Further, we will find it useful to consider only those mechanised causal graphs that are producible from a mechanised causal game satisfying \cref{assumption:ea-components-agents,assumption:opt-decision-rule,assumption:preferred-order,assumption:non-dec-term-parentless,assumption:mechanism-richness}, as this will enable us to use \cref{lem:s-reach-implies-mech-parent}.
Thus, in the next theorem, we restrict the space of mechanised causal graphs we consider. 

\begin{theorem}[\cref{alg:macid-to-xcg} is a left inverse of \cref{alg:xcg-to-macid}]
\label{thm:a_3-a_2-xcg}
    Let $\xscg$ be a mechanised causal graph such that 
    \begin{itemize}
        \item there exists a mechanised causal game, $\mecvar{\mascim}$, satisfying \cref{assumption:ea-components-agents,assumption:opt-decision-rule,assumption:preferred-order,assumption:non-dec-term-parentless,assumption:mechanism-richness}, such that $a_1(\mecvar{\mascim}) = \xscg$;
        \item any node with an incoming mechanism edge also has an incoming terminal edge, i.e.\ $\forall (\mecvar{V}, \mecvar{W}) \in E^{\textrm{mech}}, \exists (\mecvar{V'}, \mecvar{W}) \in E^{\textrm{term}}$, for some $\mecvar{V'} \in \mecvars{V}\setminus\{\mecvar{W}\}$.
    \end{itemize}
     Then $a_3(a_2(\xscg)) = \xscg$.
\end{theorem}

\begin{proof}
    The edges in $E^{\textrm{obj}}, E^{\textrm{func}}$ are the same in both ${\xscg}$ and $a_3(a_2(\xscg))$, since neither algorithm changes the object-level edges, and all mechanised causal graphs over object-level variables $\evars$ have the same edges in $E^{\textrm{func}}$, i.e. $\{(\mecvar{\evar}, \evar)\}_{\evar \in \evars}$, which are added in \cref{alg:macid-to-xcg} Lines~\ref{alg:func-loop-start}-\ref{alg:func-loop-end}. We now show why edges in $E^{\textrm{term}}$ and $E^{\textrm{mech}}$ are the same in both.
    
    From the theorem statement, $\exists \mec{\mascim}$ such that $a_1(\mec{\mascim}) = \xscg$. Let $\mascid$ be the game graph of $\mec{\mascim}$. 
    \begin{align*}
        \quad & (\mecvar{\utilvar}, \mecvar{\decisionvar}) \in E^{\textrm{term}} \textrm{ of }  \xscg
        \\
        \iff & (\mecvar{\utilvar}, \mecvar{\decisionvar}) \in E^{\textrm{term}} \textrm{ of } a_1(\mec{\mascim})
        \\
        \iff & \exists \textrm{ directed path } D \pathto U \textrm{ not through } U' \in \utilvars^A\setminus\{U\} \textrm{ in } \mascid \quad \textrm{(by \cref{lem:s-reach-implies-mech-parent})}
        \\
        \iff & \exists \textrm{ directed path } D \pathto U \textrm{ not through } U' \in \utilvars^A\setminus\{U\} \textrm{ in } a_2(a_1(\mec{\mascim})) \quad \textrm{(by \cref{thm:correctness-algo})}
        \\
        \iff & \exists \textrm{ directed path } D \pathto U \textrm{ not through } U' \in \utilvars^A\setminus\{U\} \textrm{ in } a_2(\xscg)
        \\
        \iff & (\mecvar{\utilvar}, \mecvar{\decisionvar}) \in E^{\textrm{term}}  \textrm{ of }  a_3(a_2(\xscg)) \quad \textrm{(by \cref{alg:macid-to-xcg}  Lines~\ref{alg:macid-xcg-term-start}-\ref{alg:macid-xcg-s-term-end})}.
    \end{align*}
    % \ramana{Why is it ok to apply $a_2$ in the 3rd line?}
    % \zac{hope this is addressed by spelling out reasoning, including reference to \cref{thm:correctness-algo}}
    From the theorem statement, $\forall (\mecvar{V}, \mecvar{W}) \in E^{\textrm{mech}}, \exists (\mecvar{V'}, \mecvar{W}) \in E^{\textrm{term}}$, for some $\mecvar{V'} \in \mecvars{V}\setminus\{\mecvar{W}\}$. By \cref{assumption:non-dec-term-parentless}, $W$ must be a decision in $\mec{\mascim}$.
    Thus, we only need consider edges of the form $(\mecvar{\evar}, \mecvar{D}) \in E^{\textrm{mech}}$ where $D \in \decisionvars$.
    \begin{align*}
        \quad & (\mecvar{\evar}, \mecvar{\decisionvar}) \in E^{\textrm{mech}}  \textrm{ of } \xscg
        \\
        \iff & (\mecvar{\evar}, \mecvar{\decisionvar}) \in E^{\textrm{mech}}  \textrm{ of } a_1(\mec{\mascim})
        \\
        \iff & \evar \textrm{ is s-reachable from } \decisionvar \textrm{ in } \mascid \quad \textrm{(by \cref{lem:s-reach-implies-mech-parent})}
        \\
        \iff & \evar \textrm{ is s-reachable from } \decisionvar \textrm{ in } a_2(a_1(\mec{\mascim})) \quad \textrm{(by \cref{thm:correctness-algo})}
        \\
        \iff & \evar \textrm{ is s-reachable from } \decisionvar \textrm{ in } a_2(\xscg) 
        \\
        \iff & (\mecvar{\evar}, \mecvar{\decisionvar}) \in E^{\textrm{mech}} \textrm{ of } a_3(a_2(\xscg)) \quad \textrm{(by \cref{alg:macid-to-xcg}  Lines~\ref{alg:macid-xcg-s-reach-start}-\ref{alg:macid-xcg-s-reach-end})}.
    \end{align*}
\end{proof}

\section{Examples}
\label{sec:examples}

We now look at example applications of our algorithms, which help the modeler to draw the correct game graph to describe a system. 

\subsection{Simple example}
We begin by considering the simple example of \cref{fig:mouse} in more detail. The underlying system has game graph $\mascid_{\textrm{real}}$, displayed in \cref{subfig:mouse-mascid-real}, with $D$ a decision node, $\structvar$ a chance node and $\utilvar$ a utility node.
%We take the causal game (i.e., the model underlying the diagram) to be such that 
Recall that all variables are binary; $X=D$ with probability $p$, $X=1-D$ with probability $1-p$; and $U=X$ with probability $q$, $U=1-X$ with probability $1-q$.
% That is, if $X=1$, then $U=1$ with probability $q$, else $U=0$ with probability $1-q$. If $X=0$, then $U=1$ with probability $1-q$, else $U=0$ with probability $q$.
Having specified the causal game, we can now describe the optimal decision rule -- this depends on the values of $p$ and $q$: if $p, q > 0.5$ or $p, q < 0.5$, then $D=1$ is optimal, if $p < 0.5, q > 0.5$ or $p > 0.5, q < 0.5$ then $D=0$ is optimal, and if either $p$ or $q$ is 0.5, then both $D=0$ and $D=1$ are optimal. 

We can now consider mechanism interventions, to understand what \cref{alg:loo-cd-xscm} will discover.
Suppose we soft intervene on $X$ and $U$ such that $p, q > 0.5$, so that the optimal policy is $D=1$. 
Then we change the soft intervention on $X$ such that $p < 0.5$, we will see the optimal policy change to $D=0$. Thus \cref{alg:loo-cd-xscm} draws an edge $(\mecvar{X}, \mecvar{\decisionvar})$. By a similar argument, it will also draw an edge $(\mecvar{U}, \mecvar{\decisionvar})$, which will be a terminal edge.
Thus \cref{alg:loo-cd-xscm} produces the edge-labelled mechanised causal graph $\xscg_{\textrm{model}}$ shown in \cref{subfig:mouse-xcg-model}.
\cref{alg:xcg-to-macid} then takes $\xscg_{\textrm{model}}$ and produces the correct game graph by identifying that only $\mecvar{\decisionvar}$ has incoming arrows, and so $\decisionvar$ is the only decision node, and that $\utilvar$ is the only variable which has its mechanism with an outgoing terminal edge into the mechanism for $D$, and hence is a utility. In this simple example, we have recovered the game graph of \cref{subfig:mouse-mascid-real}. 

\subsection{Optimising a model of a human}

We next consider an example from the influence diagram literature.
It has been suggested that a safety problem with content-recommendation systems is that they can nudge users towards more extreme views, to make it easier to recommend content that will generate higher utility for the system (e.g., more clicks), as the extreme views are more easily predictable \citep{benkler2018network, stray2021you, carroll2022estimating}. 
To combat this, \citet{everitt2021agent} propose that the system's utility be based on predicted clicks using a model of a user, rather than directly on actual clicks and the user's actions.
Their Fig.~4b is reproduced here in our \cref{subfig:opt-mascid}.
The node $H_1$ represents a human's initial opinion, and $H_2$ their influenced opinion after seeing an agent's recommended content, $D$. The agent observes a model of the human's initial opinion, $M$, and optimises for the predicted number of clicks, $U$, using the model $M$.

\begin{figure}
    \centering
    \begin{tabular}[c]{cc}
            \begin{subfigure}{0.45\textwidth}
        \centering
        \begin{influence-diagram}
            \node (D) [decision, player1] {$D$};
            \node (helpbelow) [below = of D, phantom] {};
            \node (helpabove) [above = of D, phantom] {};
            \node (helpleft) [left = of D, phantom] {};
            \node (helpright) [right = of D, phantom] {};
            \node (H1) [left = of helpabove] {$H_1$};
            \node (H2) [right = of helpabove] {$H_2$};
            \node (M) [left = of helpbelow] {$M$};
            \node (UD) [right = of helpbelow, utility, player1] {$\utilvar$};
            \edge {H1} {H2};
            \edge[information] {M} {D};
            \edge {D} {H2};
            \edge {D} {UD};
            \edge {M} {UD};
            \edge {H1} {M};
        \end{influence-diagram}
        \caption{$\mascid$}
        \label{subfig:opt-mascid}
    \end{subfigure}
    \begin{subfigure}{0.45\textwidth}
    %\vspace{-1cm}
        \centering
        \begin{influence-diagram}
            \node (D) [decision, player1] {$D$};
            \node (helpbelow) [below = of D, phantom] {};
            \node (helpabove) [above = of D, phantom] {};
            \node (helpleft) [left = of D, phantom] {};
            \node (helpright) [right = of D, phantom] {};
            \node (H1) [left = of helpabove] {$H_1$};
            \node (H2) [right = of helpabove] {$H_2$};
            \node (M) [left = of helpbelow] {$M$};
            \node (UD) [right = of helpbelow, utility, player1] {$\utilvar$};
            \edge {H1} {H2};
            \edge[information] {M} {D};
            \edge {D} {H2};
            \edge {D} {UD};
            \edge {M} {UD};
            \edge {H1} {M};
              \node (MH1) [left = of H1, mechanism] {$\mecvar{H_1}$};
              \node (MH2) [right = of H2, mechanism] {$\mecvar{H_2}$};
              \node (MM) [left = of M, mechanism] {$\mecvar{M}$};
              \node (MD) [right = of D, mechanism] {$\mecvar{\decisionvar}$};
              \node (MUD) [right = of UD, mechanism] {$\mecvar{\utilvar}$};
            \path (MH1) edge[function-edge] (H1);
            \path (MH2) edge[function-edge] (H2);
            \path (MM) edge[function-edge] (M);
            \path (MD) edge[function-edge] (D);
            \path (MUD) edge[function-edge] (UD);
              \path (MH2) edge[mechanism-edge, bend left=0] (MUD);
              \path (MUD) edge[terminal-edge, bend left=0] (MD);
        \end{influence-diagram}
        \caption{$\xscg$}
        \label{subfig:opt-xcg}
        \end{subfigure}
    \end{tabular}
    \caption{Recommender system optimising a model of a human. 
    \ref{subfig:opt-mascid} Game graph, $\mascid$, from \citet[Fig.~4b]{everitt2021agent}. 
    \ref{subfig:opt-xcg} mechanised causal graph, $\xscg$, that \cref{alg:loo-cd-xscm} discovers. 
    Note the path $\mecvar{H_2} \to \mecvar{U^D} \to \mecvar{D}$ which implies the recommendation system's policy depends on how a human updates their opinions when shown the recommended content, which is not visible from the game graph. 
    } 
    \label{fig:opt}
\end{figure}

Drawing the mechanised causal graph (\cref{subfig:opt-xcg}) for this system reveals some critical subtleties.
First, there is a terminal edge $(\mecvar{U}, \mecvar{D})$, since this is the goal that the agent is trained to pursue.
But should there be an edge $(\mecvar{H_2}, \mecvar{U})$?
This depends on how the user model was obtained.
If, as is common in practice, the model was obtained by predicting clicks based on past user data, then changing how a human reacts to recommended content ($\mecvar{H_2}$), would lead to a change in the way that predicted clicks depend on the model of the original user ($\mecvar{U}$).
This means that there should be an edge, as we have drawn in \cref{subfig:opt-xcg}.
\citet{everitt2021agent} likely have in mind a different interpretation, where the predicted clicks are derived from $M$ according to a different procedure, described in more detail by \citet{Farquhar2022path}.
But the intended interpretation is ambiguous when looking only at \cref{subfig:opt-mascid} -- the mechanised graph is needed to reveal the difference.

Why does all this matter? \citet{everitt2021agent} use \cref{subfig:opt-mascid} to claim that there is no incentive for the policy to instrumentally control how the human's opinion is updated and they deem the proposed system safe as a result.
However, under one plausible interpretation,  our causal discovery approach yields the mechanised causal graph representation of \cref{subfig:opt-xcg}, which contains a directed path  $\mecvar{H_2} \pathto \mecvar{D}$. 
This can be interpreted as the recommendation system is influencing the human in a goal-directed way, as it is adapting its behaviour to changes in how the human is influenced by its recommendation (cf.\ discussion in \cref{sec:other-characterisations-of-agents}).

This example casts doubt on the reliability of graphical incentive analysis \citep{everitt2021agent} and its applications \citep{Ashurst2022,Evans2021,Everitt2021tampering,Farquhar2022path,langlois2021rl,cohen2021}.
If different interpretations of the same graph yields different conclusions, then graph-based inference does not seem possible.
Fortunately, by pinpointing the source of the problem, mechanised SCMs also contain the seed of a solution: graphical incentive analysis can be trusted (only) when all non-decision mechanism lack ingoing arrows.
Indeed, this mirrors the extra assumption needed for the equivalence between games and mechanised SCMs in \cref{thm:a_3-a_2-xcg}.
As mechanisms are often assumed completely independent, this is often not an unreasonable assumption (see also \cref{sec:related-work}).
Alternatively, it may be possible to use mechanised SCMs to generalise graphical incentive analysis to allow for dependent mechanisms, but we leave investigation of this for future work.

\subsection{Actor-Critic}

Our third example contains multiple agents. It represents an Actor-Critic RL setup for a one-step MDP \citep{Sutton2018}.
Here an \emph{actor} selects action $A$ as advised by a \emph{critic} (\cref{subfig:ac-mascid-real}).
The critic's action $Q$ states the expected reward for each action (in the form of a vector with one element for each possible choice of $A$, this is often called a \emph{Q-value function}).
The action $A$ influences the state $S$, which in turn determines the reward $R$.
We model the actor as just wanting to follow the advice of the critic, so its utility is $Y=Q(A)$ (the $A$-th element of the $Q$ vector).
The critic wants its advice $Y$ to match the actual reward $R$. Formally, it optimises $W=-(R-Y)^2$.

\begin{figure}
    \centering
    \begin{tabular}[c]{cc}
    \begin{subfigure}{0.45\textwidth}
        \centering
        \begin{influence-diagram}
          \node (Q) [decision, player1] {$Q$};
          \node (Y) [right = of Q, utility, player2] {$Y$};
          \node (W) [right = of Y, utility, player1] {$W$};
          \node (A) [above = of Q, decision, player2] {$A$};
          \node (S) at (Y|-A) {$S$};
          \node (R) at (W|-A) {$R$};
          \edge {Q} {Y};
          \edge {Y} {W};
          \edge {A} {Y};
          \edge {A} {S};
          \edge {S} {R};
          \edge {R} {W};
        \end{influence-diagram}
        \caption{$\mascid_{\textrm{real}}$}
        \label{subfig:ac-mascid-real}
    \end{subfigure}
    &
    \multirow{3}{*}[14pt]{
    \begin{subfigure}{0.45\textwidth}
        \centering
        \begin{influence-diagram}
          \node (Q) {$Q$};
          \node (Y) [right = of Q] {$Y$};
          \node (W) [right = of Y] {$W$};
          \node (A) [above = of Q] {$A$};
          \node (S) at (Y|-A) {$S$};
          \node (R) at (W|-A) {$R$};
          \node (MQ) [below = of Q, mechanism] {$\mecvar{Q}$};
          \node (MY) [below = of Y, mechanism] {$\mecvar{Y}$};
          \node (MW) [below = of W, mechanism] {$\mecvar{W}$};
          \node (MA) [above = of A, mechanism] {$\mecvar{A}$};
          \node (MS) [above = of S, mechanism] {$\mecvar{S}$};
          \node (MR) [above = of R, mechanism] {$\mecvar{R}$};
          \edge {Q} {Y};
          \edge {Y} {W};
          \edge {A} {Y};
          \edge {A} {S};
          \edge {S} {R};
          \edge {R} {W};
        \path (MQ) edge[function-edge] (Q);
        \path (MY) edge[function-edge] (Y);
        \path (MW) edge[function-edge] (W);
        \path (MA) edge[function-edge] (A);
        \path (MS) edge[function-edge] (S);
        \path (MR) edge[function-edge] (R);
          \path (MQ) edge[mechanism-edge, bend left=27] (MA);
          \path (MY) edge[terminal-edge, bend left=0] (MA);
          \path (MA) edge[mechanism-edge, bend right=40] (MQ);
          \path (MS) edge[mechanism-edge, bend left=55] (MQ);
          \path (MR) edge[mechanism-edge, bend left=0] (MQ);
          \path (MW) edge[terminal-edge, bend left=25] (MQ);
          \path (MY) edge[mechanism-edge, bend left=0] (MQ);
        \end{influence-diagram}
        \caption{$\xscg_{\textrm{model}}$}
        \label{subfig:ac-xcg-model}
    \end{subfigure}
    }
    \\
    \begin{subfigure}{0.45\textwidth}
        \centering
        \begin{influence-diagram}
          \node (Q) [decision, player1] {$Q$};
          \node (Y) [right = of Q] {$Y$};
          \node (W) [right = of Y, utility, player1] {$W$};
          \node (A) [above = of Q] {$A$};
          \node (S) at (Y|-A) {$S$};
          \node (R) at (W|-A) {$R$};
          \edge {Q} {Y};
          \edge {Y} {W};
          \edge {A} {Y};
          \edge {A} {S};
          \edge {S} {R};
          \edge {R} {W};
        \end{influence-diagram}
        \caption{$\mascid_{\textrm{wrong}}$}
        \label{subfig:ac-wrong}
    \end{subfigure}\\
    \begin{subfigure}{0.45\textwidth}
        \centering
        \begin{influence-diagram}
          \node (A) [decision, player3] {$A$};
          \node (S) [right = of A] {$S$};
          \node (R) [right = of S, utility, player3] {$R$};
          \edge {A} {S};
          \edge {S} {R};
        \end{influence-diagram}
        \caption{$\mascid_{\textrm{coarse}}$}
        \label{subfig:ac-coarse}
    \end{subfigure}
    \end{tabular}
    \caption{Actor-Critic. 
    \ref{subfig:ac-mascid-real} True game graph $\mascid_{\textrm{real}}$. \ref{subfig:ac-xcg-model} 
    \cref{alg:loo-cd-xscm} produces the mechanised causal graph $\xscg_{\textrm{model}}$. From $\xscg_{\textrm{model}}$, \cref{alg:xcg-to-macid} produces the correct game graph by identifying that $\mecvar{A}$ and $\mecvar{Q}$ have incoming arrows, so are decisions, and that $Y$ has its mechanism with an outgoing terminal edge to the mechanism for $A$ so is its utility, whilst $W$ has its mechanism with an outgoing terminal edge to the mechanism for $Q$, so is its utility. They are coloured differently due to having different utilities. 
    \ref{subfig:ac-wrong} Incorrect game graph for actor-critic. \ref{subfig:ac-coarse} Coarse-grained single-agent game graph.} 
\label{fig:ac}
\end{figure}

\cref{alg:loo-cd-xscm} produces the mechanised causal graph $\xscg_{\textrm{model}}$, in \cref{subfig:ac-xcg-model}. 
We don't justify all of the mechanism edges, but instead focus on a few of interest.
For example, there is an edge $(\mecvar{S}, \mecvar{Q})$ but there is no edge $(\mecvar{S}, \mecvar{A})$, i.e.\ the critic cares about the state mechanism but the actor does not.
The critic cares because it is optimising $W$ which is causally downstream of $S$, and so the optimal decision rule for $Q$ will depend on the mechanism of $S$ even when other mechanisms are held constant.
The dependence disappears if $R$ is cut off from $S$, so the edge $(\mec{S}, \mec{Q})$ is not terminal.
In contrast, the actor \emph{doesn't} care about the mechanism of $S$, because $Y$ is \emph{not} downstream of $S$, so when holding all other mechanisms fixed, varying $\mec{S}$ won't affect the optimal decision rule for $A$.
There is however an indirect effect of the mechanism for $S$ on the decision rule for $A$, which is mediated through the decision rule for $Q$.
\cref{alg:xcg-to-macid} applied to $\xscg_{\textrm{model}}$ produces the correct game graph by identifying that $\mecvar{A}$ and $\mecvar{Q}$ have incoming arrows, and therefore are decisions; that $Y$'s mechanism has an outgoing terminal edge to $A$'s mechanism and so is its utility; and that $W$'s mechanism has an outgoing terminal edge to the mechanism for $Q$, and so is its utility.
The decision-utility subgraph consists of two connected components, one being $(A, Y)$ and the other $(Q, W)$.
The decisions and utilities therefore get coloured correctly.

This can help avoid modelling mistakes and incorrect inference of agent incentives. In particular, Christiano (private communication, 2019) has questioned the reliability of incentive analysis from CIDs, because of an apparently reasonable way of modelling the actor-critic system where the actor is not modelled as an agent, shown in \cref{subfig:ac-wrong}. 
Doing incentive analysis on this single-agent diagram would lead to the assertion that the system is not trying to influence the state $S$ or the reward $R$, because they don't lie on the directed path $Q\to W$ (i.e.\ neither $S$ nor $R$ has an \emph{instrumental control incentive}; \citealp{everitt2021agent}). This would be incorrect, as the system is trying to influence both these variables (in an intuitive and practical sense).

The modelling mistake would be avoided by applying \cref{alg:loo-cd-xscm,alg:xcg-to-macid} to the underlying system, which produce \cref{subfig:ac-mascid-real}, differing from \cref{subfig:ac-wrong}.
The correct diagram has two agents, and it's not possible to apply the single-agent incentive concept from \citep{everitt2021agent}.
Instead, an incentive concept suitable for multi-agent systems would need to be developed.
For such a multi-agent incentives concept to be useful, it should capture the influence on $S$ and $R$ jointly exerted by $A$ and $Q$.

\cref{subfig:ac-coarse} shows a game graph that involves only a subset of the variables of the underlying system, i.e., a coarse-grained version. This is also an accurate description of the same underlying system, though with less detail. At this coarser level, we find an instrumental control incentive on $S$ and $R$, as intuitively expected.

\subsection{Modified Action Markov Decision Process}
\label{subsec:mamdp}

Next, we consider an example regarding the redirectability of different RL agents.
\citet{langlois2021rl} introduce  \emph{modified action Markov decision processes} (MAMDPs)  to model a sequential decision-making problem similar to an MDP, but where the agent's decisions, $D_t$, can be overridden by a human. In the game graph in \cref{fig:mamdp}, this is modelled by $D_t$ only influencing $S_t$ via a chance variable, $X_t$, which represents the potentially overridden decision. %This intermediate variable could represent a modification made by an overseer, or some actuator noise. 

\cref{alg:loo-cd-xscm} produces the mechanised causal graph $\xscg_{\textrm{model}}$ in \cref{subfig:mamdp-xcg-model}, where for readability we restrict to mechanisms only -- for the full diagram see Fig.\ref{fig:supp-mamdp-xcg-model} in \cref{app:supp-figs}. There are many mechanism edges, so we only elaborate on the interpretation of one of the edges, $(\mecvar{X_1}, \mecvar{\decisionvar_1})$, in this mechanised causal graph. This edge represents that the agent's choice of decision rule is influenced by the mechanism for the potentially overridden variable, $X_1$. In general, in this decision problem, it will be suboptimal to ignore knowledge of the mechanism for the potentially overridden variables. 
\cref{alg:xcg-to-macid} applied to $\xscg_{\textrm{model}}$ produces the correct game graph by identifying that $\mecvar{D_1}$ and $\mecvar{D_2}$ have incoming terminal edges, so are decisions, and that $U$ has its mechanism with outgoing terminal edges to the mechanisms for decisions $D_1$ and $D_2$, and so is a utility. Since the utility is the same, the decisions are coloured the same to show they are the same agent.

\begin{figure}
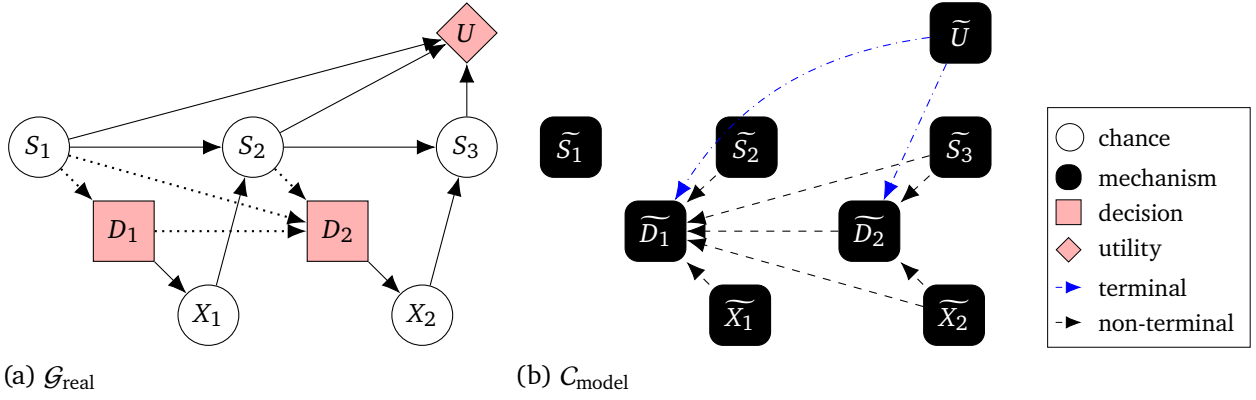

    \centering
    \begin{subfigure}[b]{0.4\textwidth}
        \centering
        \begin{influence-diagram}

          \node (S1)  {$S_1$};
    
          \node (S2) [right = 2 of S1] {$S_2$};
         
          \node (S3) [right = 2 of S2] {$S_3$};
          
          \node (help1) [below = of S1, phantom] {};
          
          \node (D1) [right = of help1, decision, player1] {$D_1$};
          
          \node (D2) [right = 2 of D1, decision, player1] {$D_2$};
          
          \node (help2) [below = of D1, phantom] {};
          
          \node (X1) [right = of help2] {$X_1$};
          
          \node (X2) [right = 2 of X1] {$X_2$};
          
          \node (U) [above = of S3, utility, player1] {$U$};

          \edge {S1} {S2};
    
          \edge {S2} {S3};
          
          \edge[information] {D1} {D2};
          
          \edge {X1} {S2};
          
          \edge {X2} {S3};
          
          \edge[information] {S1} {D1};
          
          \edge[information] {S2} {D2};
          
          \edge[information] {S1} {D2};
          
          \edge {D1} {X1};
          
          \edge {D2} {X2};
          
          \edge {S1} {U};
          
          \edge {S2} {U};
          
          \edge {S3} {U};

        \end{influence-diagram}
        \caption{$\mascid_{\textrm{real}}$}
        \label{subfig:mamdp-mascid-real}
    \end{subfigure}
    %\hfill
    \begin{subfigure}[b]{0.4\textwidth}
        \centering
        \begin{influence-diagram}

          \node (MS1) [mechanism] {$\mecvar{S_1}$};
    
          \node (MS2) [right = 1.5 of MS1, mechanism] {$\mecvar{S_2}$};
         
          \node (MS3) [right = 2 of MS2, mechanism] {$\mecvar{S_3}$};
          
          \node (help1) [below = of MS1, phantom] {};
          
          \node (MD1) [right = of help1, mechanism] {$\mecvar{D_1}$};
          
          \node (MD2) [right = 2 of MD1, mechanism] {$\mecvar{D_2}$};
          
          \node (help2) [below = of MD1, phantom] {};
          
          \node (MX1) [right = of help2, mechanism] {$\mecvar{X_1}$};
          
          \node (MX2) [right = 2 of MX1, mechanism] {$\mecvar{X_2}$};
          
          \node (MU) [above = of MS3, mechanism] {$\mecvar{U}$};

          \path (MU) edge[terminal-edge, bend right=25] (MD1);
          \path (MS2) edge[mechanism-edge, bend right=0] (MD1);
          \path (MS3) edge[mechanism-edge, bend right=0] (MD1);
          \path (MD2) edge[mechanism-edge, bend left=0] (MD1);
          \path (MX1) edge[mechanism-edge, bend right=0] (MD1);
          \path (MX2) edge[mechanism-edge, bend left=0] (MD1);
          \path (MU) edge[terminal-edge, bend right=0] (MD2);
          \path (MS3) edge[mechanism-edge, bend right=0] (MD2);
          \path (MX2) edge[mechanism-edge, bend right=0] (MD2);

        \end{influence-diagram}
        \caption{$\xscg_{\textrm{model}}$}
        \label{subfig:mamdp-xcg-model}
    \end{subfigure}
   \begin{subfigure}[b]{0.18\textwidth}
\resizebox{\textwidth}{!}{
      \begin{influence-diagram}
      \cidlegend{
   \legendrow              {}  {chance} \\
   \legendrow              {mechanism} {mechanism} \\
   \legendrow              {decision,player1}  {decision}\\
   \legendrow              {utility,player1}   {utility}\\
   \legendrow[causal]      {draw=none} {terminal} \\
   \legendrow[information] {draw=none} {non-terminal} }
 \edge[terminal-edge] {causal.west} {causal.east};
 \edge[mechanism-edge] {information.west} {information.east};
 \node [draw=none, below = of causal]  {};
     \end{influence-diagram}}
     \end{subfigure}
    \caption{Modified Action MDP. \ref{subfig:mamdp-mascid-real} The underlying system has game graph $\mascid_{\textrm{real}}$. \ref{subfig:mamdp-xcg-model} \cref{alg:loo-cd-xscm} produces the mechanised causal graph $\xscg_{\textrm{model}}$
    (we display mechanisms only, see Fig.\ref{fig:supp-mamdp-xcg-model} in \cref{app:supp-figs} for the full diagram).  
    Since the utility $U$ is the same, the decisions $D_1$ and $D_2$ are coloured the same to show they belong to the same agent.
    } 
    \label{fig:mamdp}
\end{figure}

We note that the game graph diagram presented here in \cref{subfig:mamdp-mascid-real} differs from Figure 2 of \cite{langlois2021rl}. The reason is that we have been stricter about what should appear in a game graph, and what should appear in a mechanised causal graph. In particular,  \citeauthor{langlois2021rl} have a node for the decision rule in their game graph, whereas we only have decision nodes in our game graphs, with decision rule nodes only appearing in mechanised causal graphs, along with other mechanism nodes. With this extra strictness comes greater expression and clarity -- in our game graph we are clear that the agent's decisions can't condition on the result of the modification, whereas \citeauthor{langlois2021rl} draw an information edge from the modification to the policy, which is a decision node in their diagram. 
Instead, we represent the fact that the decision rules are influenced by the mechanism for potentially overridden variables by the edges $(\mecvar{X_t}, \mecvar{\decisionvar_t})$ in the mechanised causal graph. This allows us to be clearer about what information is available for each decision (the state), but does not observe the modification, as might be construed from the diagram in \citet{langlois2021rl}.

\subsection{Zero agents}

Our final example of our algorithm working as desired is one in which there are no agents at all, see \cref{fig:zero}. Here $X$ causes $Y$ and $Y$ causes $Z$, but there is no decision or utility.
\cref{alg:loo-cd-xscm} produces the mechanised causal graph $\xscg_{\textrm{model}}$, in \cref{subfig:zero-xcg-model}.  \cref{alg:xcg-to-macid} produces the correct game graph by identifying that there are no decisions as there are no mechanisms with incoming edges, and hence also no utilities. This then just recovers a causal Bayesian network graph.

\begin{figure}
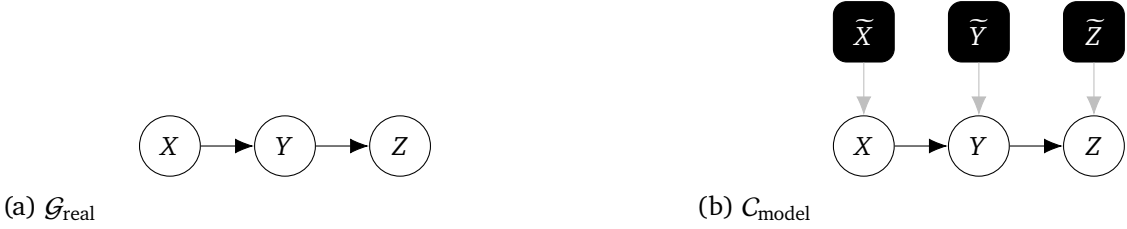

    \centering
    \begin{subfigure}[b]{0.45\textwidth}
        \centering
        \begin{influence-diagram}

          \node (X) {$X$};
    
          \node (Y) [right = of X] {$Y$};
    
          \node (Z) [right = of Y] {$Z$};

          \edge {X} {Y};
    
          \edge {Y} {Z};

        \end{influence-diagram}
        \caption{$\mascid_{\textrm{real}}$}
        \label{subfig:zero-mascid-real}
    \end{subfigure}
    \hfill
    \begin{subfigure}[b]{0.45\textwidth}
        \centering
        \begin{influence-diagram}

          \node (X) {$X$};
    
          \node (Y) [right = of X] {$Y$};
    
          \node (Z) [right = of Y] {$Z$};

          \edge {X} {Y};
    
          \edge {Y} {Z};
          
          \node (MX) [above = of X, mechanism] {$\mecvar{X}$};
    
          \node (MY) at (Y|-MX) [mechanism] {$\mecvar{Y}$};
    
          \node (MZ) at (Z|-MX) [mechanism] {$\mecvar{Z}$};

        %   \edge {MX} {X};
    
        %   \edge {MY} {Y};
          
        %   \edge {MZ} {Z};
        \path (MX) edge[function-edge] (X);
        \path (MY) edge[function-edge] (Y);
        \path (MZ) edge[function-edge] (Z);

        \end{influence-diagram}
        \caption{$\xscg_{\textrm{model}}$}
        \label{subfig:zero-xcg-model}
    \end{subfigure}
    \caption{Zero agents. \ref{subfig:zero-mascid-real} The true game graph $\mascid_{\textrm{real}}$ has no decisions or utilities, so is a standard causal Bayesian network. 
    \ref{subfig:zero-xcg-model} \cref{alg:loo-cd-xscm} produces the mechanised causal graph $\xscg_{\textrm{model}}$.  \cref{alg:xcg-to-macid} produces the correct game graph by identifying that there are no agents, and just recovers the standard causal Bayesian network.
    } 
    \label{fig:zero}
\end{figure}

\subsection{Breaking Assumptions}
\label{sec:breaking-assumptions}
Compared to the other assumptions which are more benign, \cref{assumption:ea-components-agents} rules out some examples that we might wish to consider. We now consider some examples which break it. 

\subsubsection{Multiple agents with a shared utility}
\label{sec:multiple-agents-with-a-shared-utility}

First, in \cref{fig:cirl}, we consider a causal game that has two agents with a shared utility, see \cref{subfig:cirl-mascid-real}. 
This is a diagram that represents an Assistance Game, formerly known as Cooperative Inverse Reinforcement Learning \citep{hadfield2016cooperative}. There is a human which makes decisions $H_1$ and $H_2$, conditioned on information about their preference, encoded by $\theta$, and a robot which makes decisions $R_1$ and $R_2$ based on observations of the human's decisions, but without direct observation of $\theta$. All of the human and robot decisions affect the utility, $U$ which is the same for both robot and human agents (drawn in yellow to signify that it's shared). 
This breaks \cref{assumption:ea-components-agents} because the decision-utility subgraph only has one weakly connected component, and two agent subgraphs, whereas \cref{assumption:ea-components-agents} requires the weakly connected components be the two agent subgraphs.

\begin{figure}
    \centering
    \begin{subfigure}[t]{0.3\textwidth}
        \centering
        \begin{influence-diagram}
        
          \node (help1) [phantom] {};
          
          \node (helpup) [above = 1.5 of help1, phantom] {};
          
          \node (helpdown) [below = 1.5 of help1, phantom] {};
          
          \node (helpleft) [left = 1.5 of help1, phantom] {};
          
          \node (helpright) [right = 1.5 of help1, phantom] {};

          \node (R1) at (helpleft|-helpup) [decision, player1] {$R_1$};
          
          \node (R2) at (helpright|-helpup) [decision, player1] {$R_2$};

          \node (H1) at (helpleft|-helpdown) [decision, player2] {$H_1$};
          
          \node (H2) at (helpright|-helpdown) [decision, player2] {$H_2$};

          \node (T) [below = of helpdown] {$\theta$};

          \node (U) [right = of helpright, utility, player3] {$U$};

          \edge[information] {R1}  {R2};
          
          \edge[information] {H1} {H2};
          
          \edge[information] {R1} {H2};
    
          \edge[information] {H1} {R1};
          
          \edge[information] {H2} {R2};
          
          \edge[information] {H1} {R2};
          
          \path (T) edge[->, bend right=45] (U);
          
          \edge {R1} {U};
          
          \edge {R2} {U};
          
          \edge {H1} {U};
          
          \edge {H2} {U};
          
          \edge[information] {T} {H1};
          
          \edge[information] {T} {H2};

        \end{influence-diagram}
        \caption{$\mascid_{\textrm{real}}$}
        \label{subfig:cirl-mascid-real}
    \end{subfigure}
    \hfill
    \begin{subfigure}[t]{0.38\textwidth}
        \centering
        \begin{influence-diagram}
          \node (help1) [phantom] {};
          
          \node (helpup) [above = 1.5 of help1, phantom] {};
          
          \node (helpdown) [below = 1.5 of help1, phantom] {};
          
          \node (helpleft) [left = 1.5 of help1, phantom] {};
          
          \node (helpright) [right = 1.5 of help1, phantom] {};

          \node (MR1) at (helpleft|-helpup)  [mechanism] {$\mecvar{R_1}$};
          
          \node (MR2) at (helpright|-helpup)  [mechanism] {$\mecvar{R_2}$};

          \node (MH1) at (helpleft|-helpdown)  [mechanism] {$\mecvar{H_1}$};
          
          \node (MH2) at (helpright|-helpdown)  [mechanism] {$\mecvar{H_2}$};

          \node (MT) [below = of helpdown] [mechanism] {$\mecvar{\theta}$};

          \node (MU) [right = of helpright] [mechanism] {$\mecvar{U}$};

              \path (MU) edge[terminal-edge, bend left=0] (MH1);
              \path (MH2) edge[mechanism-edge, bend left=0] (MH1);
              \path (MR1) edge[mechanism-edge, bend right=20] (MH1);
              \path (MR2) edge[mechanism-edge, bend left=10] (MH1);

              \path (MU) edge[terminal-edge, bend left=0] (MH2);
              \path (MR2) edge[mechanism-edge, bend left=20] (MH2);
              
              \path (MU) edge[terminal-edge, bend right=0] (MR1);
              \path (MH1) edge[mechanism-edge, bend right=20] (MR1);
              \path (MR2) edge[mechanism-edge, bend left=0] (MR1);
              \path (MH2) edge[mechanism-edge, bend left=0] (MR1);
              \path (MT) edge[mechanism-edge, bend left=0] (MR1);
              
              \path (MU) edge[terminal-edge, bend right=0] (MR2);
              \path (MH1) edge[mechanism-edge, bend left=10] (MR2);
              \path (MH2) edge[mechanism-edge, bend left=20] (MR2);
              \path (MT) edge[mechanism-edge, bend left=0] (MR2);

        \end{influence-diagram}
        \vspace{5mm}
        \caption{$\xscg_{\textrm{model}}$ restricted to mechanisms only}
        \label{subfig:cirl-xcg-model}
    \end{subfigure}
       \begin{subfigure}[t]{0.3\textwidth}
        \centering
        \begin{influence-diagram}
        
          \node (help1) [phantom] {};
          
          \node (helpup) [above = 1.5 of help1, phantom] {};
          
          \node (helpdown) [below = 1.5 of help1, phantom] {};
          
          \node (helpleft) [left = 1.5 of help1, phantom] {};
          
          \node (helpright) [right = 1.5 of help1, phantom] {};

          \node (R1) at (helpleft|-helpup) [decision, player1] {$R_1$};
          
          \node (R2) at (helpright|-helpup) [decision, player1] {$R_2$};

          \node (H1) at (helpleft|-helpdown) [decision, player1] {$H_1$};
          
          \node (H2) at (helpright|-helpdown) [decision, player1] {$H_2$};
          
        %   \node (help1) [below = of H1, phantom] {};
          
        %   \node (help2) [right = of H1, phantom] {};
          
          \node (T) [below = of helpdown] {$\theta$};

          \node (U) [right = of helpright, utility, player1] {$U$};

          \edge[information] {R1} {R2};
          
          \edge[information] {H1} {H2};
          
          \edge[information] {R1} {H2};
    
          \edge[information] {H1} {R1};
          
          \edge[information] {H2} {R2};
          
          \edge[information] {H1} {R2};
          
          \path (T) edge[->, bend right=45] (U);
          
          \edge {R1} {U};
          
          \edge {R2} {U};
          
          \edge {H1} {U};
          
          \edge {H2} {U};
          
          \edge[information] {T} {H1};
          
          \edge[information] {T} {H2};

        \end{influence-diagram}
        \caption{$\mascid_{\textrm{model}}$}
        \label{subfig:cirl-mascid-model}
    \end{subfigure}
    \caption{Assistance Game (A.K.A. CIRL). \ref{subfig:cirl-mascid-real} True game graph $\mascid_{\textrm{real}}$, where the yellow utility indicates both robot and human share the same utility. 
    \ref{subfig:cirl-xcg-model} \cref{alg:loo-cd-xscm} produces the mechanised causal graph $\xscg_{\textrm{model}}$, shown here restricted to mechanisms only for readability -- see \cref{app:supp-figs}, \cref{fig:supp-cirl-xcg-model} for the full mechanised causal graph.  
    \ref{subfig:cirl-mascid-model}
    \cref{alg:xcg-to-macid} produces the an incorrect game graph in this case, because we violated \cref{assumption:ea-components-agents}, and gives that all decisions belong to the same agent.} 
    \label{fig:cirl}
\end{figure}

\cref{alg:loo-cd-xscm} produces the mechanised causal graph $\xscg_{\textrm{model}}$ shown in \ref{subfig:cirl-xcg-model}. We believe this mechanised causal graph representation of the system is correct. However, the problem arises when we apply \cref{alg:xcg-to-macid} on it. The result is the game graph in \ref{subfig:cirl-mascid-model} which has all decisions belonging to the same agent. We hope that in future work we will be able to distinguish agents which share the same utility, through some modification to the colouring logic of \cref{alg:xcg-to-macid}. One approach may be to use some condition involving sufficient recall \citep{milch2008ignorable} to distinguish between agents (a game graph has sufficient recall if for all agents the mechanism graph restricted to that agent's decision rules is acyclic).

\subsubsection{Non-descendent utility}

We now consider an example, \cref{subfig:ndu-mascid-real}, that breaks \cref{assumption:ea-components-agents} in another way. There are two agents which make decisions $A$ and $B$ with utilities $U^A$ and $U^B$ respectively. The red agent chooses $A$, which affects the utility that the blue agent receives $U^B$. The blue agent's choice affects the red agent's utility. Note that the agent subgraph for $A$ is disconnected (no directed path from $A$ to $U^A$), so this example violates \cref{assumption:ea-components-agents}. 

\cref{alg:loo-cd-xscm} applied to $\mascid_{\textrm{real}}$ produces the mechanised causal graph $\xscg_{\textrm{model}}$, in \cref{subfig:ndu-xcg-model}.
We think this mechanised causal graph is an accurate representation of the system. From inspecting it, we can see that although $U^A$ is not a descendent of $A$, it is a descendent of $\mecvar{A}$, via $\mecvar{A} \rightarrow \mecvar{B} \rightarrow B \rightarrow U^A$. That is, the red agent's decision rule can still have an effect on its utility, but \cref{assumption:opt-decision-rule,assumption:preferred-order} rule out agents strategising using this path.
Applying \cref{alg:xcg-to-macid} on $\xscg_{\textrm{model}}$ produces an incorrect game graph with $A$ and $\utilvar^A$ being incorrectly identified as chance nodes (\cref{subfig:ndu-mascid-model}).

This example highlights several questions for future work:
%\begin{itemize}
Which agents learn to influence their utility by means of their decision rule, thereby breaking our \cref{assumption:opt-decision-rule,assumption:preferred-order}? 
% Compare human behaviour in the Ultimatum Game \citep{Harsanyi1961}.
And how can \cref{alg:xcg-to-macid} be generalised to handle non-descendant utilities and agents utilising influence from their decision rule?
%\end{itemize}

\begin{figure}
    \centering
    \begin{subfigure}[b]{0.3\textwidth}
        \centering
        \begin{influence-diagram}

          \node (A) [decision, player1] {$A$};
          
          \node (B) [below = of A, decision, player2] {$B$};

          \node (UB) [right = of A, utility, player2] {$U^B$};

          \node (UA) [right = of UB, utility, player1] {$U^A$};

          \edge {A} {UB};
    
          \edge {B} {UB};
          
          \edge {B} {UA};

        \end{influence-diagram}
        \caption{$\mascid_{\textrm{real}}$}
        \label{subfig:ndu-mascid-real}
    \end{subfigure}
    \hfill
    \begin{subfigure}[b]{0.37\textwidth}
        \centering
        \begin{influence-diagram}

          \node (A) {$A$};
          
          \node (B) [below = of A, ] {$B$};

          \node (UB) [right = of A] {$U^B$};

          \node (UA) [right = of UB, ] {$U^A$};

          \edge {A} {UB};
    
          \edge {B} {UB};
          
          \edge {B} {UA};
          
          \node (MA) [above = of A, mechanism] {$\mecvar{A}$};
          \node (MB) [left = of B, mechanism] {$\mecvar{B}$};
          \node (MUB) [below = of UB, mechanism] {$\mecvar{U^B}$};
          \node (MUA) [above = of UA, mechanism] {$\mecvar{U^A}$};

            \path (MA) edge[function-edge] (A);
            \path (MB) edge[function-edge] (B);
            \path (MUB) edge[function-edge] (UB);
            \path (MUA) edge[function-edge] (UA);

              \path (MA) edge[mechanism-edge, bend left=0] (MB);
              \path (MUB) edge[terminal-edge, bend left=25] (MB);

        \end{influence-diagram}
        \vspace{-3mm}
        \caption{$\xscg_{\textrm{model}}$}
        \label{subfig:ndu-xcg-model}
    \end{subfigure}
       \begin{subfigure}[b]{0.3\textwidth}
        \centering
        \begin{influence-diagram}

          \node (A) {$A$};
          
          \node (B) [below = of A, decision, player2] {$B$};

          \node (UB) [right = of A, utility, player2] {$U^B$};

          \node (UA) [right = of UB] {$U^A$};

          \edge {A} {UB};
    
          \edge {B} {UB};
          
          \edge {B} {UA};

        \end{influence-diagram}
        \caption{$\mascid_{\textrm{model}}$}
        \label{subfig:ndu-mascid-model}
    \end{subfigure}
    \caption{Non-descendant utility. \ref{subfig:ndu-mascid-real} True game graph $\mascid_{\textrm{real}}$. Note that the agent subgraph  for $A$ (\cref{def:agent-subgraph}) is not connected, violating \cref{assumption:ea-components-agents}. \ref{subfig:ndu-xcg-model} \cref{alg:loo-cd-xscm} produces the mechanised causal graph $\xscg_{\textrm{model}}$. 
    \ref{subfig:ndu-mascid-model}
    \cref{alg:xcg-to-macid} produces an incorrect game graph in this case, because we violated \cref{assumption:ea-components-agents}, leading to $A$ and $\utilvar^A$ being incorrectly identified as chance, rather than as decision and utility variables respectively.} 
    \label{fig:ndu}
\end{figure}

\section{Discussion}
\label{sec:discussion}

\subsection{Relativism of variable types}
\label{subsec:relativism}
The first thing we discuss is that the variable types in a causal game, i.e. decision, utility or chance, are only meaningful relative to the choice of which variables are included in the model. Whether our procedure of \cref{alg:loo-cd-xscm} followed by \cref{alg:xcg-to-macid} classifies a variable as decision, utility or chance depends on what other variables are included in the graph. For example, if in reality there is a utility variable $\utilvar$ which is not present in the model (i.e., the set of variables doesn't include $\utilvar$), but some of its parents, $\Pa^\utilvar$, are present in the model, then those parents will be labelled as utilities. Similarly, if in reality there is a decision variable, $\decisionvar$, which is not present in the model, but some of its children, $\Ch^\decisionvar$, are present in the model, then those children will be labelled as decisions. See \cref{app:relativism-example} for a simple example of this relativism.
In a sense, a choice of variables represents a frame in which to model the system, and what is a decision or a utility node is frame-dependent.

\subsection{Modelling advice}
% \section{modelling advice}
\label{subsec:modelling-advice}

How does one identify the relevant variables to begin with?
\cref{alg:loo,alg:loo-cd-xscm,alg:xcg-to-macid} only provides a way to determine the structure of a mechanised SCM and associated game graph from a given set of variables, but not how to choose them.
We now offer some tips on choosing variables.

A few principles always apply.
First, variables should represent aspects of the environment that we are concerned with, either as means of influence for an agent, or as inherently valuable aspects of the environment. 
The content selected by a content recommender system, and the preferences of a user, are good examples.
Second, it should be fully clear both how to measure and how to intervene on a variable.
Otherwise its causal relationship to other variables will be ill-defined.
In our case, this requirement extends also to the mechanism of each variable.
Third, a variable's domain should be exhaustive (cover all possible outcomes of that variable) and represent mutually exclusive events (no pair of outcomes can occur at the same time) \citep{kjaerulff2008bayesian}.
Finally, variables should be \emph{logically independent}: one variable taking on a value should never be mutually exclusive with another variable taking on a particular value \citep{Halpern2010}.

It's important to clarify whether a variable is object-level or a mechanism. For example, previous work \citep{langlois2021rl} has drawn a policy (i.e., a mechanism) in a way that makes it look like an object-level variable, which led to some confusion, whereas in \cref{subsec:mamdp} we take the decision rule to be a mechanism.
Another lesson learnt is that there are important differences between a utility and a variable which is merely instrumental for that utility. This is evident when performing a structural mechanism intervention to cut off instrumental variables from their downstream utilities, in which case a decision-maker won't respond to changes only in the instrumental variable.

Of particular importance is the level of coarse-graining in the choice of variables. There are some works on the marginalisation of Bayesian networks \citep{evans2016graphs,kinney2020causal}, and in cyclic SCMs \citep{bongers2016foundations}, which allow for one to marginalise out some variables. We hope to explore marginalisation in the context of game graphs  in future work, and present one example in \cref{sec:app-marg-merge}. The choice of coarse-graining may have an impact on whether agents are discovered.

\subsection{Relationship to Causality Literature}
\label{sec:related-work}

We now discuss some related literature in Causality. Other related work was discussed in \cref{sec:other-characterisations-of-agents}.
\cite{pearl2009causality} lays the foundations for modern approaches to causality, with emphasis on graphical models, and in particular through the use of structural causal models (SCMs), which allow for treatment of both interventions and counterfactuals.  
\cite{dawid2002} considers related approaches to causal modelling, including the use of influence diagrams to specify  which variables can be intervened on. One model that's introduced is called a \emph{parameter DAG}, which is similar to our mechanised SCM, in that each object-level variable has a \emph{parameter} variable which parametrises the distribution of the object-level variable. However, whilst acknowledging there could be links between the parameter variables, they are not considered in that work.
In contrast, our focus is less on using influence diagrams as a tool for causal modelling, and rather on modelling and discovering agents using causal methods. Further, we allow relationships between mechanism variables in our models, and elucidate their relation to decision, chance and utility variables in the influence diagram representation.

% \textcolor{orange}{[Cyclic models]}
%Whilst \cite{pearl2009causality} treats acyclic models,
\citet{halpern2000axiomatizing} gives an axiomatization of SCMs, generalizing to cases where the structural equations may not have a unique (or any) solution.
However, in the case of non-unique (or non-existant) solutions, potential response variables are ill-defined, which \cite{white2009settable} claim prevents the desired causal discourse. They instead propose the \emph{settable systems} framework in which there are \emph{settable} variables which have a role-indicator argument which determines whether a variable's value is a \emph{response}, determined by its structural equation, or if its value is a \emph{setting}, as determined by a hard intervention.
\cite{bongers2016foundations} give formalizations for statistical causal modelling using cyclic SCMs, proving certain properties present in the acyclic case don't hold in the cyclic case.
In our work, we use mechanised SCMs that can have cycles between mechanism variables. Zero, one or multiple solutions reflect the multiple equilibria arising in some games. Our formalism for mechanised SCMs follows the cyclic SCM treatment of \cite{bongers2016foundations}. 

% \textcolor{orange}{[Soft Interventions]}
\cite{correa2020calculus} develop \emph{sigma-calculus} for reasoning about the identification of the effects of soft interventions using observational, rather than experimental data. In our work, we assume access to experimental data, which makes the identification question trivial. Future work could relax this assumption to explore when agents can be discovered from observational data.
Their \emph{regime indicators} roughly correspond to our mechanism variables.

\label{sec:causal-discovery-review}

Our work draws on structure discovery in the causal discovery literature. See \cite{glymour2019review} for a review, and 
\cite{forre2018constraint} for an example of causal discovery of cyclic models.
The usual focus in causal discovery is not to model agents, but rather to model some physical (agent-agnostic) system (modelling agents is usually done in the context of decision/game theory). Our work differs in that we use causal discovery in order to get a causal model representation of agents (a mechanised SCM), and can then translate that to the game-theoretic description in terms of game graphs with agents.

One of the most immediate applications of our results concerns the independent causal mechanisms (ICM) principle \citep{scholkopf2012causal,peters2017elements,scholkopf2021toward}. ICM states that,

\begin{enumerate}
    \item Changing (intervening on) the causal mechanism $M^X$ for $P(X\mid\mathbf{pa}^X)$ does not change any of the other mechanisms $M^Y$ for $P(Y\mid\mathbf{pa}^Y)$, $X \neq Y$.
    \item Knowledge of $M^X$ does not provide knowledge of $M^Y$ for any $X \neq Y$.
\end{enumerate}

ICM argues that $P(X\mid\mathbf{pa}^X)$ typically describes fixed and modular causal mechanisms that do not respond to the mechanisms of other variables. The classic example is the distribution of atmospheric temperature $T$ given its causes $\mathbf{pa}^T$ such as altitude. While the distribution $P(\mathbf{pa}^T)$ may vary between countries, $P(T\mid\mathbf{pa}^T)$ remains fixed as it describes a physical law relating altitude (and other causes) to atmospheric temperature. In recent years ICM has become the predominant inductive bias used in causal machine learning including causal and disentangled representations \citep{bengio2013representation,locatello2019challenging,scholkopf2022causality}, causal discovery \citep{janzing2012information,janzing2010causal}, semi-supervised learning \citep{scholkopf2012causal}, adversarial vulnerability \citep{schott2018towards}, reinforcement learning \citep{bengio2019meta}, and has even played a role in major scientific discoveries such as discovering the first exoplanet with atmospheric water \citep{foreman2015systematic}. 
Our results provide a constraint on the applicability of the ICM principle; namely that 

\begin{quote}
    $P(X\mid\mathbf{pa}^X)$ does not obey the ICM principle if it is an agent's decision rule, 
    or is strategically relevant to some agent's decision rule, 
    as determined by \cref{alg:xcg-to-macid}.
\end{quote}

Condition 1 in the ICM is true only if $X$ is not strategically relevant for an agent, and condition 2 covers agents themselves, as their mechanisms are correlated with the mechanisms of strategically relevant variables. This limits the applicability of ICM (and methods based on ICM) to systems where the data generating process includes no agents. Likely examples include sociological data and data generated by reinforcement learning agents during training. However, our hope is that \cref{alg:loo-cd-xscm} can be applied to identify mechanism edges that violate ICM, allowing ICM to be applied to the correct systems, and in doing so improve the performance of ICM-based methods.

\section{Conclusion}

We proposed the first formal causal definition of agents.
Grounded in causal discovery, our key contribution is to formalise the idea that agents are systems that adapt their behaviour in response to changes in how their actions influence the world.
Indeed, \cref{alg:loo-cd-xscm,alg:xcg-to-macid} describe a precise experimental process that can, in principle and under some assumptions, be done to assess whether something is an agent.
Our process is largely consistent with previous, informal characterisations of agents (e.g.\ \citealp{dennett1987intentional,Wiener1961,flint2020the,garrabrant2021saving}), but making it formal 
enables agents and their incentives to be identified empirically or from the system architecture.
Our process improves upon an earlier formalisation by \citet{orseau2018agents}, by better handling systems with a small number of actions and "accidentally optimal" systems (see \cref{sec:other-characterisations-of-agents} for details).

Causal modelling of AI systems is a tool of growing importance, and this paper grounds this area of work in causal discovery experiments.
We have demonstrated the utility of our approach by improving the safety analysis of several AI systems (see \cref{sec:examples}).
In consequence, this would also improve the reliability of methods building on such modelling, such as analyses of the safety and fairness of machine learning algorithms (see e.g.\ \citealp{Ashurst2022,Evans2021,Everitt2021tampering,Farquhar2022path,langlois2021rl,cohen2021,richens2022counterfactual}).

\section*{Acknowledgement}
We thank Laurent Orseau, Mary Phuong, James Fox, Lewis Hammond, Francis Rhys Ward, and Ryan Carey for comments and discussions. %\url{https://docs.google.com/presentation/d/1lFf5YpPZXfWLyEDrYwtnYKSz7h8XuSBt_DPU2eyJgZk/edit#slide=id.g11b92422e30_0_0}

\bibliography{main}

\newpage
\appendix

\section{Mathematical Background}
\label{app:math-background}

\counterwithin{definition}{section}

\subsection{Notation}
We use roman capital letters $V$ for variables, lower case for their outcomes $v$. We use bold type to indicate vectors of variables, $\evars$, and vectors of outcomes $\evals$. 
Parent, children, ancestor and descendent variables are denoted $\Pa^\evar, \Ch^\evar, \Anc^\evar, \Desc^\evar$, respectively, with the family denoted by $\Fa^\evar = \Pa^\evar \cup \{\evar\}$. 
% \zac{could cut above two sentences as repeat from Background, but for flow and ease of reading of this section in isolation have repeated here for now.}
We use $\dom(V)$ and $\dom(\evars) = \times_{V\in \evars} \dom(V)$ to denote the set of possible outcomes of $V$ and $\evars$ respectively, which are assumed finite.
Subscripts are reserved for denoting submodels and potential responses to an intervention.

\subsection{Structural Causal Model}

We begin with a standard definition of a structural causal model.

\begin{definition}[Structural Causal Model (SCM), \citealp{pearl2009causality}]
    \label{def:scm}
    A \emph{structural causal model (SCM)} is given by the tuple $\scm = \langle \evars, \exovars^{\evars}, \structfns, \Pr(\exovars^{\evars}) \rangle$ where
    \begin{itemize}
        \item $\evars$ is a set of endogenous variables.
        \item $\exovars^{\evars} = \{\exovar^{\evar}\}_{\evar \in \evars}$ is a set of exogenous variables, one for each endogenous variable.
        \item $\structfns = \{\evar = \structfn^{\evar}(\evars, \exovar^\evar)\}_{\evar \in \evars}$ is a set of structural equations, one for each endogenous variable.
        \item $\Pr(\exovars^{\evars})$ is a distribution over the exogenous variables. 
     \end{itemize}
\end{definition}

This has an associated directed graph, called a causal graph (CG).
\begin{definition}[Causal Graph (CG)]
    \label{def:cg}
    For an SCM, $\scm = \langle \evars, \exovars^{\evars}, \structfns, \Pr(\exovars^{\evars}) \rangle$, a \emph{causal graph} (CG) is the directed graph $\scg = \langle\evars, E\rangle$, where the set of directed edges, $E$, represent endogenous dependencies in the set of structural equations $\structfns$, so that 
    $(\evar, \evaralt) \in E$ if and only if $\evaralt, \evar \in \evars$ and $f^\evaralt(\evars, \exovar^\evar)$ depends on the value of $\evar$ (as such, all our causal graphs are faithful \citep{pearl2009causality} by construction).
\end{definition}
 The subgraph of white nodes in \cref{subfig:mouse-xcg-model} is an example of a CG.

In some parts of this work, we will consider acyclic (recursive) SCMs, in which the CG is acyclic. 
Other parts will consider cases in which there is a possibly cyclic (nonrecursive) SCM, in which the CG is cyclic.
See \cite{bongers2016foundations} for a foundational treatment of  SCMs with cyclic CGs.
They define a solution of an SCM as a set of exogenous and endogenous random variables, $\exovars, \evars$, for which the exogenous distribution matches that in the cyclic SCM, and for which the structural equations are satisfied. For a solution, $\exovars, \evars$, the distribution over the endogenous variables, $\Pr^{\evars}$ is called the observational distribution associated to $\evars$.
In this cyclic case, there can be zero, one or many observational distributions, due to the existence of different solutions of the structural equations.
In this work, we assume the existence of a unique solution, even in the case of a nonrecursive SCM. 
This unique solution then defines a joint distribution over endogenous variables \citep{pearl2009causality}
\begin{align}
    \Pr^{[\scm]}(\evars = \evals) =  \sum_{\{\exovals | \evars(\exovals)= \evals \}} \Pr(\exovals).
\end{align}

SCMs model causal interventions that set variables to particular outcomes, captured by the following definition of a submodel:
\begin{definition}[SCM Submodel, \citealp{pearl2009causality}]
    \label{def:submodel}
    Let $\scm = \langle \evars, \exovars^{\evars}, \structfns, \Pr(\exovars^{\evars}) \rangle$ be an SCM, $\sY \subseteq \evars$ 
    be a set of endogenous variables, and $\sy \in \dom(\sY)$ a value for each variable in that subset.
    The submodel $\scm_{\sy}$ represents the effects of an \emph{intervention} $\doo(\sY=\sy)$ and is formally defined as the SCM $\scm_{\sy} = \langle \evars, \exovars^{\evars}, \structfns_{\sy}, \Pr(\exovars^{\evars}) \rangle$ where $\structfns_{\sy} = \{\evar = \structfn^{\evar}(\Pa^\evar, \exovar^\evar)\}_{\evar \in \evars\setminus\sY} \cup \{\sY = \sy\}$. That is, the original functional relationships for $\sY$ are replaced with the constant functions $\sY = \sy$.
\end{definition}

We also assume the existence of a unique solution to the set of structural equations under all interventions, allowing us to define the potential response.
\begin{definition}[Potential Response, \citealp{pearl2009causality}]
    \label{def:potential-response}
    Let $\scm = \langle \evars, \exovars^{\evars}, \structfns, \Pr(\exovars^{\evars}) \rangle$ be an SCM, and let $\sX, \sY \subseteq \evars$. The \emph{potential response} of $\sX$ to the intervention $\doo(\sY=\sy)$, denoted $\sX_\sy(\exovars)$ is the solution for $\sX$ in the set of equations $\structfns_{\sy}$, that is, $\sX_\sy(\exovars) = \sX_{\scm_{\sy}}(\exovars)$, where $\scm_{\sy}$ is the submodel from intervention $\doo(\sY=\sy)$.
\end{definition}

\subsection{Structural Causal Game}

We now introduce a (structural) causal game, which draws on the SCM, emphasising the structural causal dependencies present.
\begin{definition}[Structural Causal Game, \citealp{hammond2021reasoning,everitt2021agent}]
    \label{def:scim}
    A (Markovian) (structural) causal game is a tuple 
    \begin{align*}
    \mec{\mascim}= \langle N, \evars, \exovars^{\evars}, \{\Pa^{\decisionvar}\}^{\decisionvar \in \decisionvars}, \structfns, \Pr(\exovars^{\evars}) \rangle    
    \end{align*}
     where 
    \begin{itemize}
        \item $N=\{1, \dots, n\}$ is a set of agents
        \item $\evars = \decisionvars \cup \structvars \cup \utilvars$ is a set of endogenous variables, partitioned into decision, chance and utility variables respectively. 
        \item $\exovars^{\evars} = \{\exovar^{\evar}\}_{\evar \in \evars}$ is a set of exogenous  variables, one for each endogenous variable
         
        \item $\{\Pa^{\decisionvar}\}^{\decisionvar \in \decisionvars}$ is a set of information parents for each decision variable $\decisionvar$, with $\Pa^\decisionvar \subseteq \evars\setminus\decisionvar$
        \item $\structfns = \{\evar = \structfn^{\evar}(\Pa^\evar, \exovar^\evar)\}_{\evar \in \evars \setminus \decisionvars}$ is a set of structural equations for each non-decision endogenous variable, as specified by the functions
        $\structfn^{\evar}: \dom(\Pa^\evar \cup \{ \exovar^\evar\}) \mapsto \dom(\evar)$, where $\Pa^\evar \subseteq \evars \setminus \{\evar\}$.
        \item $\Pr(\exovars^{\evars})$ is a distribution over the exogenous variables such that the individual exogenous variables are mutually independent.
    \end{itemize}
\end{definition}

The causal game has an associated graph:
\begin{definition}[Game Graph]
    \label{def:cid}
    Let $\mec{\mascim} =  \langle N, \evars, \exovars^{\evars}, \{\Pa^{\decisionvar}\}^{\decisionvar \in \decisionvars}, \structfns, \Pr(\exovars^{\evars})  \rangle $ be a causal game. 
    We define the \emph{game graph} to be the structure $\mascid = (N, \evars\cup \exovars^{\evars}, E)$, where $N=\{1, \dots, n\}$ is a set of agents and $(\evars\cup \exovars^{\evars}, E)$ is a DAG with:
    \begin{itemize}
        \item Four vertex types $\evars\cup \exovars^{\evars} = \structvars \cup \utilvars  \cup \decisionvars \cup \exovars^{\evars}$: the first three types are endogenous nodes in white circles, coloured diamonds and squares respectively; the fourth type are exogenous nodes, $\exovars^{\evars}$, in grey circles. The different colours of diamonds and squares correspond to different agents.
        \item Two types of edges:
        \begin{itemize}
            \item dependence edges, $(\evar, \evaralt) \in E$ if and only if either $\evaralt \in \evars\setminus\decisionvars$ and $\evar$ is an argument to the structural function $\structfn^{\evaralt}$, i.e. $\evar \in \Pa^{\evaralt} \cup \exovars^{\evaralt}$; or $\evaralt = \decisionvar \in \decisionvars$ and $\evar = \exovar_{\decisionvar}$.
            These are denoted with solid edges.
            \item information edges, $(\evar, \decisionvar) \in E$ if and only if $\evar \in \Pa^\decisionvar$ of the causal game. These are denoted with dashed edges.
        \end{itemize}
    \end{itemize}
\end{definition}

One can also draw a simpler graph by omitting the exogenous variables and their outgoing edges from the game graph.
\cref{subfig:mouse-mascid-real} is an example of a game graph.
We will only consider causal games for which the associated game graph is acyclic.

For each non-decision variable, the causal game specifies a distribution over it. For the decision variables, the causal game doesn't specify how it is distributed, only the information available at the time of the decision, as captured by $\Pa^\decisionvar$. The agents get to select their behaviour at each of their decision nodes, as follows.
Let $\mascim
$ be a causal game. 
A \emph{decision rule}, $\decisionrule^\decisionvar$, for a decision variable $\decisionvar \in \decisionvars^i \subseteq \decisionvars$ is a (measurable) structural function $\decisionrule^\decisionvar: \dom(\Pa^\decisionvar \cup \{\exovar^\decisionvar\}) \mapsto \dom(\decisionvar)$ where $\exovar^\decisionvar$ is uniformly distributed over the $[0, 1]$ interval.%
\footnote{For settings where we are interested in arbitrary counterfactual queries, a more complex form of  $\exovar^\decisionvar$ has advantages \citep{hammond2021reasoning}.}
A \emph{partial policy profile}, $\bm{\decisionrule}^{\decisionvars '}$ is a set of decision rules $\decisionrule^\decisionvar$ for each $\decisionvar \in \decisionvars ' \subseteq \decisionvars$.
A \emph{policy} refers to $\bm{\decisionrule}^{a}$, the set of decision rules for all of agent $a$'s decisions.
A \emph{policy profile}, $\bm{\decisionrule} = (\bm{\decisionrule}^{1}, \dots, \bm{\decisionrule}^{n})$ assigns a decision rule to every agent.

For a causal game $\mascim$, we can combine its set of structural equations $\structfns = \{\evar = \structfn^{\evar}(\Pa^\evar, \exovar^\evar)\}_{\evar \in \evars \setminus \decisionvars}$ with a policy profile $\bm{\decisionrule}$ to obtain a \emph{Policy-game SCM}, $\mascim(\bm{\decisionrule}) =  \langle \evars, \exovars^{\evars}, \structfns^{\bm{\decisionrule}}, \Pr(\exovars^{\evars}) \rangle$ with the set of structural equations $\structfns^{\bm{\decisionrule}} = \structfns \cup \{\decisionvar= \decisionrule^{\decisionvar}(\Pa^\decisionvar, \exovar^\decisionvar)\}_{\decisionvar \in \decisionvars}$. 
Note that there is a well-defined endogenous distribution, $\Pr^{[\mascim(\bm{\decisionrule})]}$, as the policy-game SCM is acyclic, due to the game graph being a DAG.

Each agent's expected utility in policy profile $\bm{\decisionrule}$ is given by
\begin{align}
    \label{eq:eu}
    EU^a(\bm{\decisionrule}) = \sum_{\{\utilval^{a} \in \dom(\utilvar^{a})\}}\Pr^{[\mascim(\bm{\decisionrule})]}(\utilvar^{a} = \utilval^{a}) \cdot \utilval^{a}.
\end{align}

\begin{definition}[Optimality and Best Response, \citealp{Koller2003-yh}]
    \label{def:optimality}
    Let $\bm{k} \subseteq \decisionvars^a$ and let $\bm{\decisionrule}$ be a policy profile. We say that partial policy profile $\hat{\bm{\decisionrule}}^{\bm{k}}$ is \emph{optimal for policy profile} $\bm{\decisionrule} = (\bm{\decisionrule}^{-\bm{k}}, \hat{\bm{\decisionrule}}^{\bm{k}})$ if in the induced causal game $\mascim(\bm{\decisionrule}^{-\bm{k}})$, where the only remaining decisions are those in $\bm{k}$, the decision rule $\hat{\bm{\decisionrule}}^{\bm{k}}$ is optimal, i.e. for all partial policy profiles $\bm{{\decisionrule}^{\bm{k}}}$
    \begin{align}
        EU^a((\bm{\decisionrule}^{-\bm{k}}, \hat{\bm{\decisionrule}}^{\bm{k}})) \geq EU^a((\bm{\decisionrule}^{-\bm{k}}, \bm{{\decisionrule}^{\bm{k}}})).
    \end{align}
    Agent $a$'s decision rule $\bm{{\decisionrule}}^{a}$ is a \emph{best response} to the partial policy profile $\bm{{\decisionrule}}^{-a}$ assigning strategies to the decisions of all other agents if for all strategies $\bm{{\decisionrule}}^{a}$
    \begin{align}
        EU^a((\bm{{\decisionrule}}^{-a}, \hat{\bm{{\decisionrule}}}^{a})) \geq EU^a((\bm{{\decisionrule}}^{-a}, \bm{{\decisionrule}}^{a})).
    \end{align}
\end{definition}

In the game-theoretic setting with multiple agents, we typically consider rational behaviour to be represented by a \emph{Nash Equilibrium}:

\begin{definition}[Nash Equilibrium, \citealp{Koller2003-yh}]
    \label{def:nash}
    A policy profile $\bm{\decisionrule} = (\bm{\decisionrule}^{1}, \dots, \bm{\decisionrule}^{n})$ is a \emph{Nash Equilibrium} if for all agents $a$, $\bm{\decisionrule}^{a}$ is a best response to $\bm{\decisionrule}^{-a}$.
\end{definition}

In this paper, we consider the refined concept of subgame perfect equilibrium (SPE), as follows
\begin{definition}[Subgame Perfect Equilibrium, \citealp{hammond2021equilibrium, hammond2021reasoning}]
    \label{def:spe}
    A policy profile $\bm{\decisionrule} = (\bm{\decisionrule}^{1}, \dots, \bm{\decisionrule}^{n})$ is a \emph{Subgame Perfect Equilibrium} if for all subgames, $\bm{\decisionrule}$ is a Nash equilibrium.
\end{definition}
Informally, in any subgame, the rational response is independent of variables outside of the subgame.  
See \cite{hammond2021reasoning, hammond2021equilibrium} for the formal definition of a subgame in a causal game.

\section{Getting a mechanised SCM by Marginalisation and Merging}
\label{sec:app-marg-merge}

A bandit algorithm repeatedly chooses an arm $X$ and receives a reward $U$. We can represent two iterations by using indices.
\begin{center}
    \begin{influence-diagram}
    \node (X1) {$X^1$};
    \node (U1) [right = of X1] {$U^1$};
    \node (X2) [right = of U1] {$X^2$};
    \node (U2) [right = of X2] {$U^2$};
    \edge {X1} {U1}
    \edge {X2} {U2}
    \end{influence-diagram}
\end{center}
If we include the mechanisms in the graph, we can model the fact that the policy at time 2, i.e.\ $\mecvar{X^2}$, depends on the arm and outcome at time 1:
\begin{center}
    \begin{influence-diagram}
    \node (X1) {$X^1$};
    \node (U1) [right = of X1] {$U^1$};
    \node (X2) [right = of U1] {$X^2$};
    \node (U2) [right = of X2] {$U^2$};
    \edge {X1} {U1}
    \edge {X2} {U2}
    \node (fX1) [mechanism, below = of X1] {$\mecvar{X^1}$};
    \node (fU1) [mechanism, below = of U1] {$\mecvar{U^1}$};
    \node (fX2) [mechanism, below = of X2] {$\mecvar{X^2}$};
    \node (fU2) [mechanism, below = of U2] {$\mecvar{U^2}$};
    \edge {fX1} {X1}
    \edge {fU1} {U1}
    \edge {fX2} {X2}
    \edge {fU2} {U2}
    \path (fX1) edge[mechanism-edge, bend right] (fX2);
    \path (fU1) edge[mechanism-edge, bend right] (fU2);
    \edge{X1,U1} {fX2};
    \end{influence-diagram}
\end{center}
To arrive at the final mechanism graph, we first marginalise $X^1$ and $U^1$. The path from $\mecvar{U^1}$ to $\mecvar{X^2}$ previously mediated by $U^1$ now becomes a direct edge.
\begin{center}
    \begin{influence-diagram}
    \node (X2) [] {$X^2$};
    \node (U2) [right = of X2] {$U^2$};
    \edge {X2} {U2}
    \node (fX2) [mechanism, below = of X2] {$\mecvar{X^2}$};
    \node (fU2) [mechanism, below = of U2] {$\mecvar{U^2}$};
    \node (fU1) [mechanism, left = of fX2] {$\mecvar{U^1}$};
    \node (fX1) [mechanism, left = of fU1] {$\mecvar{X^1}$};
    \path (fU1) edge[mechanism-edge] (fX2);
    \edge {fX2} {X2}
    \edge {fU2} {U2}
    \path (fX1) edge[mechanism-edge, bend right] (fX2);
    \path (fU1) edge[mechanism-edge, bend right] (fU2);
    \end{influence-diagram}
\end{center}
Finally, we merge $\mecvar{X^1}$ with $\mecvar{X^2}$ and $\mecvar{U^1}$ with $\mecvar{U^2}$, with the understanding that observing the merged node $\mecvar{X}$ corresponds to observing $\mecvar{X^2}$, while intervening on $\mecvar{X}$ means setting both $\mecvar{X^1}$ and $\mecvar{X^2}$. This yields the following mechanised causal graph (note the terminal edge due to $U^2$ not having any children)
\begin{center}
    \begin{influence-diagram}
    \node (X2) [] {$X^2$};
    \node (U2) [right = of X2] {$U^2$};
    \edge {X2} {U2}
    \node (fX2) [mechanism, below = of X2] {$\mecvar{X}$};
    \node (fU2) [mechanism, below = of U2] {$\mecvar{U}$};
    \edge[terminal-edge] {fU2} {fX2}
    \edge {fX2} {X2}
    \edge {fU2} {U2}
    \end{influence-diagram}
\end{center}
Applying \cref{alg:xcg-to-macid} yields the expected game graph:
\begin{center}
    \begin{influence-diagram}
    \node (X2) [decision, player1] {$X^2$};
    \node (U2) [right = of X2, utility, player1] {$U^2$};
    \edge {X2} {U2}
    \end{influence-diagram}
\end{center}

\section{Example of relativism of variable types}
\label{app:relativism-example}

This example illustrates the discussion of Sec.\ref{subsec:relativism} with an example of the relativism of variable types.
Whether a variable gets classified as a decision, chance, or utility node by \cref{alg:xcg-to-macid} depends on which other nodes are included in the graph.
To see this, consider the graph in \cref{fig:thermometer} in which a blueprint  for a thermometer, $B$, influences the constructed thermometer, $T$, and thereby whether the reading is correct or not, $C$.

\begin{figure}
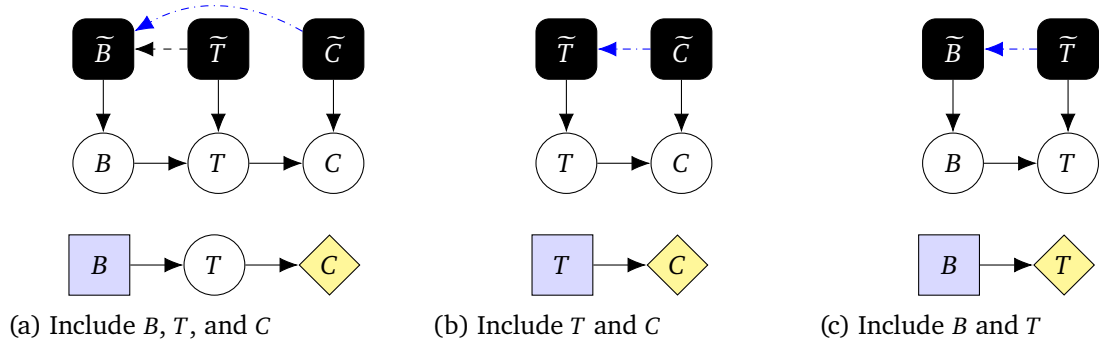

\centering
    \begin{subfigure}[b]{0.33\textwidth}
        \centering
        \begin{influence-diagram}
        \node (B) {$B$};
        \node (T) [right = of B] {$T$};
        \node (C) [right = of T] {$C$};
        \edge {B}{T}
        \edge {T}{C}
        \node (fB) [above = of B, mechanism] {$\mecvar{B}$};
        \node (fT) [above = of T, mechanism] {$\mecvar{T}$};
        \node (fC) [above = of C, mechanism] {$\mecvar{C}$};
        \edge {fB} {B}
        \edge {fT} {T}
        \edge {fC} {C}
        \edge[mechanism-edge] {fT} {fB}
        \path (fC) edge[->, terminal-edge, bend right](fB);
        \end{influence-diagram}\\[0.5cm]
        
        \begin{influence-diagram}
        \node (B) [decision]{$B$};
        \node (T) [right = of B] {$T$};
        \node (C) [right = of T, utility] {$C$};
        \edge {B}{T}
        \edge {T}{C}
        \end{influence-diagram}
        \caption{Include $B$, $T$, and $C$}
        \label{subfig:btc}
    \end{subfigure}
    \begin{subfigure}[b]{0.3\textwidth}
        \centering
        \begin{influence-diagram}
        \node (T) [] {$T$};
        \node (C) [right = of T] {$C$};
        \edge {T}{C}
        \node (fT) [above = of T, mechanism] {$\mecvar{T}$};
        \node (fC) [above = of C, mechanism] {$\mecvar{C}$};
        \edge {fT} {T}
        \edge {fC} {C}
        \edge[terminal-edge] {fC} {fT}
        \end{influence-diagram}\\[0.5cm]
        
        \begin{influence-diagram}
        \node (T) [decision] {$T$};
        \node (C) [right = of T, utility] {$C$};
        \edge {T}{C}
        \end{influence-diagram}
        \caption{Include $T$ and $C$}
        \label{subfig:tc}
    \end{subfigure}
    \begin{subfigure}[b]{0.3\textwidth}
        \centering
        \begin{influence-diagram}
        \node (B) {$B$};
        \node (T) [right = of B] {$T$};
        \edge {B}{T}
        \node (fB) [above = of B, mechanism] {$\mecvar{B}$};
        \node (fT) [above = of T, mechanism] {$\mecvar{T}$};
        \edge {fB} {B}
        \edge {fT} {T}
        \edge[terminal-edge] {fT} {fB}
        \end{influence-diagram}\\[0.5cm]
        
        \begin{influence-diagram}
        \node (B) [decision]{$B$};
        \node (T) [right = of B, utility] {$T$};
        \edge {B}{T}
        \end{influence-diagram}
        \caption{Include $B$ and $T$}
        \label{subfig:bt}
    \end{subfigure}
    \caption{What is a decision or a utility node depends on what other variables are included. Here the variables represent a blueprint for a thermometer (B), the constructed thermometer (T), and thereby whether the reading is correct or not (C)}
    \label{fig:thermometer}
\end{figure}

Considering first \cref{subfig:btc} where a first modeler has included all three variables. We find that the designer will produce a different blueprint if they are aware that blueprints are interpreted according to a different convention (i.e., if $\mecvar{T}$ changes), or if temperature was measured at a different scale (a change to $\mecvar{C}$).
Accordingly, \cref{alg:xcg-to-macid} labels $B$ a decision, and $C$ a utility.
This makes sense in this context: the designer chooses a blueprint to ensure that the thermometer gives a correct reading.

A second modeler may not care about the blueprint, and only wonder about the relationship between the produced thermometer $T$ and the correctness of the reading $C$. See \cref{subfig:tc}. 
They will find that if temperature was measured at a different scale, a slightly different thermometer would have been produced, i.e.\ $\mecvar{C}$ now influences $\mecvar{T}$ rather than $\mecvar{B}$ (as in \cref{subfig:btc}).
This is not a contradiction, as $\mecvar{T}$ is a different object in \cref{subfig:btc}  and~\ref{subfig:tc}. In \cref{subfig:btc}, $\mecvar{T}$ represents the relationship between $B$ and $T$, while in \cref{subfig:tc}, $\mecvar{T}$ represents the marginal distribution of $T$.
As a consequence, \cref{alg:xcg-to-macid} will label $T$ as a decision optimising $C$.
This is not unreasonable: a decision was made to produce a particular kind of thermometer with the aim of getting correct temperature readings.

A third modeler may not bother to represent the correctness of the readings explicitly, and only consider the blueprint and the produced thermometer, see \cref{subfig:bt}.
They will find that the blueprint is optimised to obtain a particular kind of thermometer.
Again, this is not unreasonable, as in this context we may well speak of the designer deciding on a blueprint that will produce the right kind of thermometer.

\section{Supplementary Figures}
\label{app:supp-figs}

\begin{figure}
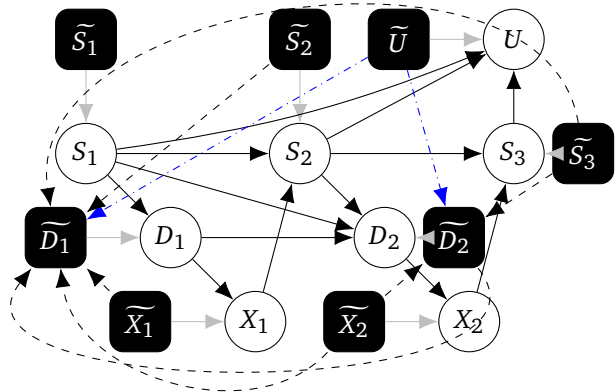

        \centering
        \begin{influence-diagram}

          \node (S1)  {$S_1$};
    
          \node (S2) [right = 2 of S1] {$S_2$};
         
          \node (S3) [right = 2 of S2] {$S_3$};
          
          \node (help1) [below = of S1, phantom] {};
          
          \node (D1) [right = of help1] {$D_1$};
          
          \node (D2) [right = 2 of D1] {$D_2$};
          
          \node (help2) [below = of D1, phantom] {};
          
          \node (X1) [right = of help2] {$X_1$};
          
          \node (X2) [right = 2 of X1] {$X_2$};
          
          \node (U) [above = of S3] {$U$};

          \node (MS1) [above = of S1, mechanism] {$\mecvar{S_1}$};
    
          \node (MS2) [above = of S2, mechanism] {$\mecvar{S_2}$};
         
          \node (MS3) [right = 0.1 of S3, mechanism] {$\mecvar{S_3}$};
          
          \node (MD1) [left = of D1, mechanism] {$\mecvar{D_1}$};
          
          \node (MD2) [right = 0.1 of D2, mechanism] {$\mecvar{D_2}$};
          
          \node (MX1) [left = of X1, mechanism] {$\mecvar{X_1}$};
          
          \node (MX2) [left = of X2, mechanism] {$\mecvar{X_2}$};
          
          \node (MU) [left = of U, mechanism] {$\mecvar{U}$};

          \path (MS1) edge[function-edge] (S1);
          \path (MS2) edge[function-edge] (S2);
          \path (MS3) edge[function-edge] (S3);
          \path (MD1) edge[function-edge] (D1);
          \path (MD2) edge[function-edge] (D2);
          \path (MX1) edge[function-edge] (X1);
          \path (MX2) edge[function-edge] (X2);
          \path (MU) edge[function-edge] (U);

          \edge {S1} {S2};
    
          \edge {S2} {S3};
          
          \edge {D1} {D2};
          
          \edge {X1} {S2};
          
          \edge {X2} {S3};
          
          \edge {S1} {D1};
          
          \edge {S2} {D2};
          
          \edge {S1} {D2};
          
          \edge {D1} {X1};
          
          \edge {D2} {X2};
          
          \path (S1) edge[->, bend right=7] (U);
          
          \edge {S2} {U};
          
          \edge {S3} {U};
        
          \path (MU) edge[terminal-edge, bend right=0] (MD1);
          \path (MS2) edge[mechanism-edge, bend right=0] (MD1);
          \path (MS3) edge[mechanism-edge, bend right=90] (MD1);
          \path (MD2) edge[mechanism-edge, bend left=128] (MD1);
          \path (MX1) edge[mechanism-edge, bend right=0] (MD1);
          \path (MX2) edge[mechanism-edge, bend left=65] (MD1);
          \path (MU) edge[terminal-edge, bend right=0] (MD2);
          \path (MS3) edge[mechanism-edge, bend right=0] (MD2);
          \path (MX2) edge[mechanism-edge, bend right=0] (MD2);

        \end{influence-diagram}
        \caption{Full mechanised causal graph for MAMDP example in \cref{fig:mamdp}}
        \label{fig:supp-mamdp-xcg-model}
    \end{figure}

\begin{figure}
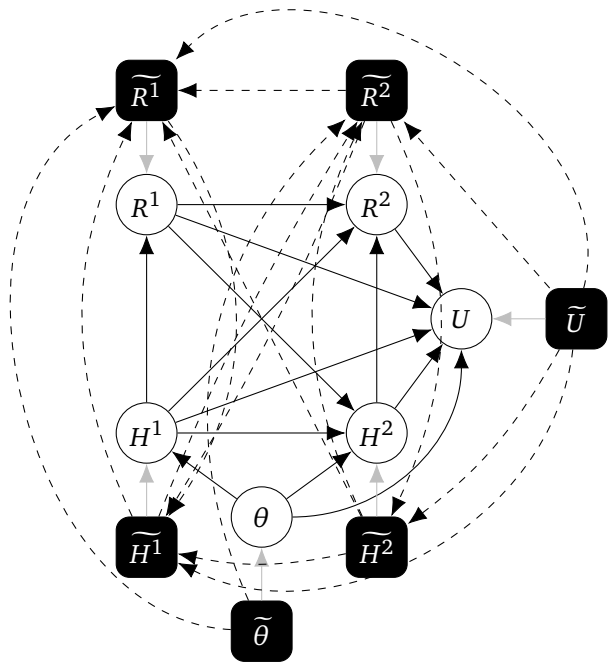

        \centering
        \begin{influence-diagram}
          \node (help1) [phantom] {};
          
          \node (helpup) [above = 1.5 of help1, phantom] {};
          
          \node (helpdown) [below = 1.5 of help1, phantom] {};
          
          \node (helpleft) [left = 1.5 of help1, phantom] {};
          
          \node (helpright) [right = 1.5 of help1, phantom] {};

          \node (R1) at (helpleft|-helpup)  {$R^1$};
          
          \node (R2) at (helpright|-helpup)  {$R^2$};

          \node (H1) at (helpleft|-helpdown)  {$H^1$};
          
          \node (H2) at (helpright|-helpdown)  {$H^2$};

          \node (T) [below = of helpdown] {$\theta$};

          \node (U) [right = of helpright] {$U$};

          \edge {R1} {R2};
          
          \edge {H1} {H2};
          
          \edge {R1} {H2};
    
          \edge {H1} {R1};
          
          \edge {H2} {R2};
          
          \edge {H1} {R2};
          
          \path (T) edge[->, bend right=45] (U);
          
          \edge {R1} {U};
          
          \edge {R2} {U};
          
          \edge {H1} {U};
          
          \edge {H2} {U};
          
          \edge {T} {H1};
          
          \edge {T} {H2};
              
          \node (MH1) [below = of H1, mechanism] {$\mecvar{H^1}$};
          \node (MH2) [below = of H2, mechanism] {$\mecvar{H^2}$};
          \node (MR1) [above = of R1, mechanism] {$\mecvar{R^1}$};
          \node (MR2) [above = of R2, mechanism] {$\mecvar{R^2}$};
          \node (MT) [below = of T, mechanism] {$\mecvar{\theta}$};
          \node (MU) [right = of U, mechanism] {$\mecvar{U}$};

            \path (MH1) edge[function-edge] (H1);
            \path (MH2) edge[function-edge] (H2);
            \path (MR1) edge[function-edge] (R1);
            \path (MR2) edge[function-edge] (R2);
            \path (MT) edge[function-edge] (T);
            \path (MU) edge[function-edge] (U);

              \path (MU) edge[mechanism-edge, bend left=55] (MH1);
              \path (MH2) edge[mechanism-edge, bend left=12] (MH1);
              \path (MR1) edge[mechanism-edge, bend left=35] (MH1);
              \path (MR2) edge[mechanism-edge, bend left=5] (MH1);

              \path (MU) edge[mechanism-edge, bend left=12] (MH2);
              \path (MR2) edge[mechanism-edge, bend left=25] (MH2);
              
              \path (MU) edge[mechanism-edge, bend right=75] (MR1);
              \path (MH1) edge[mechanism-edge, bend left=25] (MR1);
              \path (MR2) edge[mechanism-edge, bend left=0] (MR1);
              \path (MH2) edge[mechanism-edge, bend left=0] (MR1);
              \path (MT) edge[mechanism-edge, bend left=75] (MR1);
              
              \path (MU) edge[mechanism-edge, bend right=0] (MR2);
              \path (MH1) edge[mechanism-edge, bend left=5] (MR2);
              \path (MH2) edge[mechanism-edge, bend left=25] (MR2);
              \path (MT) edge[mechanism-edge, bend left=35] (MR2);

        \end{influence-diagram}
        \caption{Full mechanised causal graph for Assistance Game example in \cref{fig:cirl}}
        \label{fig:supp-cirl-xcg-model}
    \end{figure}

\end{document}